\documentclass[12pt]{article}

\usepackage[numbers, compress]{natbib}
\usepackage[utf8]{inputenc} %
\usepackage[T1]{fontenc}    %
\usepackage{hyperref}       %
\usepackage{url}            %
\usepackage{booktabs}       %
\usepackage{amsfonts}       %
\usepackage{nicefrac}       %
\usepackage{microtype}      %
\usepackage{xcolor}         %

\usepackage{amsmath, nccmath, amssymb}
\usepackage{amsthm}
\usepackage[ruled]{algorithm2e}
\usepackage{cleveref}
\usepackage{graphicx}
\usepackage{tcolorbox}
\usepackage{mathrsfs}
\usepackage{graphbox}
\usepackage{caption}
\usepackage{subcaption}
\usepackage{wrapfig}
\usepackage{float}
\usepackage{enumitem}
\usepackage{tabu}
\usepackage{doi}
\usepackage{times}
\usepackage[parfill]{parskip}
\usepackage[letterpaper,
            left=1.25in,
            right=1.25in,
            top=1in,
            bottom=1in]{geometry}

\usepackage{amsmath, amssymb, amsthm, amsfonts}
\usepackage{graphicx}
\usepackage{enumitem}
\usepackage{cleveref}
\usepackage{xcolor}
\usepackage{mathrsfs}
\usepackage{physics}
\usepackage{accents}

\DeclareMathOperator*{\argmin}{arg\,min}
\newcommand{\proj}{\mathrm{proj}}

\newcommand{\thetas}{\theta^\star}

\newcommand{\thetad}{\theta^\dagger}

\newcommand{\vs}{\accentset{\star}{v}}
\newcommand{\xs}{\accentset{\star}{x}}
\newcommand{\ys}{\accentset{\star}{y}}

\renewcommand{\bar}{\overline}

\DeclareMathOperator{\step}{step}

\newtheorem{lemma}{Lemma}
\newtheorem{corollary}{Corollary}
\newtheorem{theorem}{Theorem}

\newtheorem{assumption}{Assumption}

\newtheorem{definition}{Definition}

\title{Self-Stabilization: The Implicit Bias of Gradient Descent at the Edge of Stability}

\author{%
  Alex Damian* \\
  Princeton University\\
  \texttt{ad27@princeton.edu}
  \and
  Eshaan Nichani* \\
  Princeton University \\
  \texttt{eshnich@princeton.edu} \\
  \and
  Jason D. Lee \\
  Princeton University \\
  \texttt{jasonlee@princeton.edu}
}

\ifdefined\usebigfont

\usepackage{times}
\usepackage[fontsize=13pt]{scrextend}
\AtBeginDocument{\newgeometry{letterpaper,left=1.56in,right=1.56in,top=1.71in,bottom=1.77in}}
\else
\fi

\begin{document}

\maketitle
\begin{NoHyper}
\def\thefootnote{*}\footnotetext{Equal contribution}
\end{NoHyper}

\begin{abstract}
Traditional analyses of gradient descent show that when the largest eigenvalue of the Hessian, also known as the sharpness $S(\theta)$, is bounded by $2/\eta$, training is "stable" and the training loss decreases monotonically. Recent works, however, have observed that this assumption does not hold when training modern neural networks with full batch or large batch gradient descent. Most recently, \citet{cohen2021eos} observed two important phenomena. The first, dubbed \emph{progressive sharpening}, is that the sharpness steadily increases throughout training until it reaches the instability cutoff $2/\eta$. The second, dubbed \emph{edge of stability}, is that the sharpness hovers at $2/\eta$ for the remainder of training while the loss continues decreasing, albeit non-monotonically. 

We demonstrate that, far from being chaotic, the dynamics of gradient descent at the edge of stability can be captured by a cubic Taylor expansion: as the iterates diverge in direction of the top eigenvector of the Hessian due to instability, the cubic term in the local Taylor expansion of the loss function causes the curvature to decrease until stability is restored. This property, which we call \emph{self-stabilization}, is a general property of gradient descent and explains its behavior at the edge of stability.
A key consequence of self-stabilization is that gradient descent at the edge of stability implicitly follows \emph{projected} gradient descent (PGD) under the constraint $S(\theta) \le 2/\eta$. Our analysis provides precise predictions for the loss, sharpness, and deviation from the PGD trajectory throughout training, which we verify both empirically in a number of standard settings and theoretically under mild conditions. Our analysis uncovers the mechanism for gradient descent's implicit bias towards stability.

\end{abstract}

\section{Introduction}
\subsection{Gradient Descent at the Edge of Stability}
Almost all neural networks are trained using a variant of gradient descent, most commonly stochastic gradient descent (SGD) or ADAM \citep{kingma2015adam}. When deciding on an initial learning rate, many practitioners rely on intuition drawn from classical optimization. In particular, the following classical lemma, known as the "descent lemma," provides a common heuristic for choosing a learning rate in terms of the sharpness of the loss function:
\begin{definition}
    Given a loss function $L(\theta)$, the sharpness is defined to be $S(\theta) := \lambda_{max}(\nabla^2 L(\theta))$. When this eigenvalue is unique, the associated eigenvector is denoted by $u(\theta)$.
\end{definition}
\begin{lemma}[Descent Lemma]
    Assume that $S(\theta) \le \ell$ for all $\theta$. If $\theta_{t+1} = \theta_t - \eta \nabla L(\theta_t)$,
    \begin{align*}
        L(\theta_{t+1}) \le L(\theta_t) - \frac{\eta \qty(2 - \eta \ell)}{2} \norm{\nabla L(\theta_t)}^2.
    \end{align*}
\end{lemma}
Here, the loss decrease is proportional to the squared gradient, and is controlled by the quadratic $\eta (2 - \eta \ell)$ in $\eta$. This function is maximized at $\eta = 1/\ell$, a popular learning rate criterion. For any $\eta < 2/\ell$, the descent lemma guarantees that the loss will decrease. As a result, learning rates below $2/\ell$ are considered "stable" while those above $2/\ell$ are considered "unstable." For quadratic loss functions, e.g. from linear regression, this is tight. Any learning rate above $2/\ell$ provably leads to exponentially increasing loss.

However, it has recently been observed that in neural networks, the descent lemma is not predictive of the optimization dynamics. Recently, \citet{cohen2021eos} observed two important phenomena for gradient descent, which made more precise similar observations in \citet{jastrzebski2019on,Jastrzebski2020The} for SGD:

\paragraph{Progressive Sharpening} Throughout most of the optimization trajectory, the gradient of the loss is negatively aligned with the gradient of sharpness, i.e. $\nabla L(\theta) \cdot \nabla S(\theta) < 0.$ As a result, for any reasonable learning rate $\eta$, the sharpness increases throughout training until it reaches $S(\theta) = 2/\eta$.

\paragraph{Edge of Stability} Once the sharpness reaches $2/\eta$ (the ``break-even'' point in \citet{Jastrzebski2020The}), it ceases to increase and remains around $2/\eta$ for the rest of training. The descent lemma no longer guarantees the loss decreases but the loss still continues decreasing, albeit non-monotonically.

\begin{figure}[ht]
    \centering
    \includegraphics[width=\textwidth]{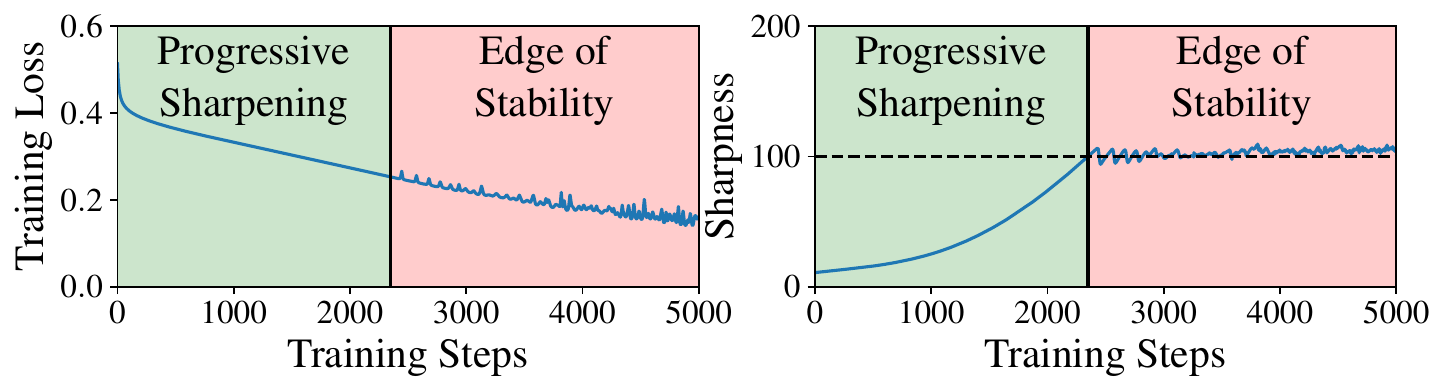}
    \caption{\textbf{Progressive Sharpening and Edge of Stability:} We train an MLP on CIFAR10 with learning rate $\eta = 2/100$. It reaches instability after around $2200$ training steps after which the sharpness hovers at $2/\eta = 100$, which is denoted by the horizontal dashed line.}
    \label{fig:basicEOS}
\end{figure}

\subsection{Self-stabilization: The Implicit Bias of Instability}

In this work we explain the second stage, "edge of stability." We identify a new implicit bias of gradient descent which we call \textit{self-stabilization}. Self-stabilization is the mechanism by which the sharpness remains bounded around $2/\eta$, despite the continued force of progressive sharpening, and by which the gradient descent dynamics do not diverge, despite instability. Unlike progressive sharpening, which is only true for specific loss functions (eg. those resulting from neural network optimization \citep{cohen2021eos}), self stabilization is a general property of gradient descent.

Traditional non-convex optimization analyses involve Taylor expanding the loss function to second order around $\theta$ to prove loss decrease when $\eta \le 2/S(\theta)$. When this is violated, the iterates diverge exponentially in the top eigenvector direction, $u$, thus leaving the region in which the loss function is locally quadratic. Understanding the dynamics thus necessitates a \emph{cubic} Taylor expansion.

Our key insight is that the missing term in the Taylor expansion of the gradient after diverging in the $u$ direction is $\nabla^3 L(\theta)(u,u)$, which is conveniently equal to the gradient of the sharpness at $\theta$:
\begin{lemma}[Self-Stabilization Property]\label{lem:nabla_eig}
    If the top eigenvalue of $\nabla^2 L(\theta)$ is unique, then the sharpness $S(\theta)$ is differentiable at $\theta$ and $\nabla S(\theta) = \nabla^3 L(\theta)(u(\theta),u(\theta))$.
\end{lemma}
As the iterates move in the negative gradient direction, this term has the effect of \emph{decreasing the sharpness}. The story of self-stabilization is thus that as the iterates diverge in the $u$ direction, the strength of this movement in the $-\nabla S(\theta)$ direction grows until it forces the sharpness below $2/\eta$, at which point the iterates in the $u$ direction shrink and the dynamics re-enter the quadratic regime. 

This negative feedback loop prevents both the sharpness $S(\theta)$ and the movement in the top eigenvector direction, $u$, from growing out of control. As a consequence, we show that gradient descent \emph{implicitly} solves the \emph{constrained minimization problem}:
\begin{align}
    \min_\theta L(\theta) \qqtext{such that} S(\theta) \le 2/\eta.
\end{align}
Specifically, if the stable set is defined by $\mathcal{M} := \qty{\theta ~:~ S(\theta) \le 2/\eta \text{ and } \nabla L(\theta) \cdot u(\theta) = 0}$\footnote{The condition that $\nabla L(\theta) \cdot u(\theta) = 0$ is added to ensure the constrained trajectory is not unstable. This condition does not affect the stationary points of \cref{eq:constrained_update}.} then the gradient descent trajectory $\{\theta_t\}$ tracks the following projected gradient descent (PGD) trajectory $\{\thetad_t\}$ which solves the constrained problem \citep{barber2017pgd}:
\begin{align}\label{eq:constrained_update}
    \thetad_{t+1} = \proj_\mathcal{M}\qty(\thetad_t - \eta \nabla L(\thetad_t)) \qq{where} \proj_\mathcal{M}\qty(\theta) := \argmin_{\theta' \in \mathcal{M}} \norm{\theta - \theta'}.
\end{align}

Our main contributions are as follows. First, we explain self-stabilization as a generic property of gradient descent for a large class of loss functions, and provide precise predictions for the loss, sharpness, and deviation from the constrained trajectory $\{\thetad_t\}$ throughout training (\Cref{sec:heuristic}). Next, we prove that under mild conditions on the loss function (which we verify empirically for standard architectures and datasets), our predictions track the true gradient descent dynamics up to higher order error terms (\Cref{sec:theory}). Finally, we verify our predictions by replicating the experiments in \citet{cohen2021eos} and show that they model the true gradient descent dynamics (\Cref{sec:experiments}).

\section{Related Work}
\citet{Xing2018AWW} observed that for some neural networks trained by full-batch gradient descent, the loss is not monotonically decreasing. \citet{wu2018selectminimizer} remarked that gradient descent cannot converge to minima where the sharpness exceeds $2/\eta$ but did not give a mechanism for avoiding such minima. \citet{Lewkowycz2020} observed that when the initial sharpness is larger than $2/\eta$, gradient descent "catapults" into a stable region and eventually converges. \citet{jastrzebski2019on} studied the sharpness along stochastic gradient descent trajectories and observed an initial increase (i.e. progressive sharpening) followed by a peak and eventual decrease. They also observed interesting relationships between the dynamics in the top eigenvector direction and the sharpness. \citet{Jastrzebski2020The} conjectured a general characterization of stochastic gradient descent dynamics asserting that the sharpness tends to grow but cannot exceed a stability criterion given by their eq (1), which reduces to $S(\theta) \le 2/\eta$ in the case of full batch training. \citet{cohen2021eos} demonstrated that for the special case of (full batch) gradient descent training, the optimization dynamics exhibit a simple characterization. First, the sharpness rises until it reaches $S(\theta) = 2/\eta$ at which point the dynamics transition into an ``edge of stability'' (EOS) regime where the sharpness oscillates around $2/\eta$ and the loss continues to decrease, albeit non-monotonically.

Recent works have sought to provide theoretical analyses for the EOS phenomenon. \citet{Ma2022TheMS} analyzes EOS when the loss satisfies a "subquadratic growth" assumption. \citet{Ahn2022UnderstandingTU} argues that unstable convergence is possible when there exists a "forward invariant subset" near the set of minimizers. \citet{Arora2022EoS} analyzes progressive sharpening and the EOS phenomenon for normalized gradient descent close to the manifold of global minimizers. \citet{Lyu2022Normalization} uses the EOS phenomenon to analyze the effect of normalization layers on sharpness for scale-invariant loss functions.
\citet{Chen2022EoS} show global convergence despite instability for certain 2D toy problems and in a 1-neuron student-teacher setting. The concurrent work \citet{Li2022EoS} proves progressive sharpening for a two-layer network and analyzes the EOS dynamics through four stages similar to ours using the norm of the output layer as a proxy for sharpness. 

Beyond the EOS phenomenon itself, prior work has also shown that SGD with large step size or small batch size will lead to a decrease in sharpness \citep{keskar2017, Jastrzebski2017ThreeFI, jastrzebski2019on, Jastrzebski2020The}. \citet{gilmer2022} also describes connections between EOS, learning rate warm-up, and gradient clipping.

At a high level, our proof relies on the idea that oscillations in an unstable direction prescribed by the quadratic approximation of the loss cause a longer term effect arising from the third-order Taylor expansion of the dynamics. This overall idea has also been used to analyze the implicit regularization of SGD~\citep{BlancGVV20, damian2021label, li2022what}. In those settings, oscillations come from the stochasticity, while in our setting the oscillations stem from instability.

\section{Setup}\label{sec:setup}

We denote the loss function by $L \in C^3(\mathbb{R}^d)$. Let $\theta \in \mathbb{R}^d$ follow gradient descent with learning rate $\eta$, i.e. $\theta_{t+1} := \theta_t - \eta \nabla L(\theta_t)$. Recall that 
\begin{align*}
    \mathcal{M} := \qty{\theta ~:~ S(\theta) \le 2/\eta \text{ and } \nabla L(\theta) \cdot u(\theta) = 0}
\end{align*}
is the set of stable points and $\proj_\mathcal{M} := \argmin_{\theta' \in \mathcal{M}} \norm{\theta - \theta'}$ is the orthogonal projection onto $\mathcal{M}$. For notational simplicity, we will shift time so that $\theta_0$ is the first point such that $S(\proj_\mathcal{M}(\theta)) = 2/\eta$. 

As in \cref{eq:constrained_update}, the constrained trajectory $\thetad$ is defined by 
\begin{align*}
    \thetad_0 := \proj_\mathcal{M}(\theta_0) \qand \thetad_{t+1} := \proj_{\mathcal{M}}(\thetad_t - \eta \nabla L(\thetad_t)).
\end{align*}

Our key assumption is the existence of progressive sharpening along the constrained trajectory, which is captured by the progressive sharpening coefficient $\alpha(\theta)$:
\begin{definition}[Progressive Sharpening Coefficient]\label{def:alpha}
    We define $\alpha(\theta) := -\nabla L(\theta) \cdot \nabla S(\theta)$.
\end{definition}

\begin{assumption}[Existence of Progressive Sharpening]\label{assumption:progressive_sharpening}
    $\alpha(\thetad_t) > 0$.
\end{assumption}

We focus on the regime in which there is a single unstable eigenvalue, and we leave understanding multiple unstable eigenvalues to future work. We thus make the following assumption on $\nabla^2 L(\thetad_t)$:

\begin{assumption}[Eigengap]\label{assumption:eigval_gap}
    For some absolute constant $c < 2$ we have $\lambda_2(\nabla^2 L(\thetad_t)) < c/\eta$.
\end{assumption}

\section{The Self-stabilization Property of Gradient Descent}\label{sec:heuristic}

\begin{figure}
    \centering
    \includegraphics[width=\textwidth]{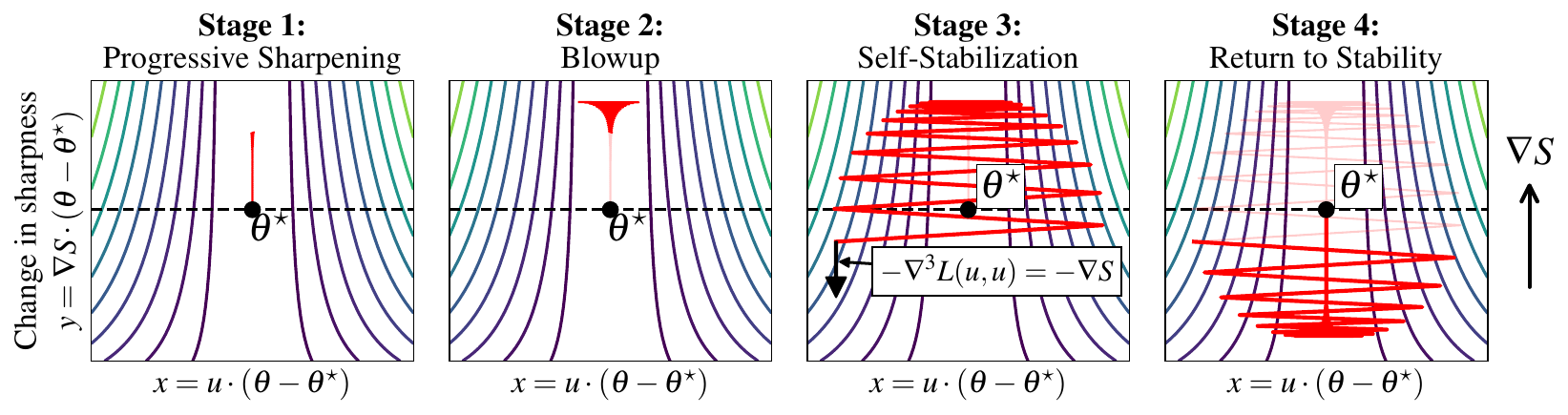}
    \caption{The four stages of edge of stability (see \Cref{sec:four_stages}), demonstrated on a simple loss function (see \Cref{sec:toy_model}).}
    \label{fig:stages}
\end{figure}

In this section, we derive a set of equations that predict the displacement between the gradient descent trajectory $\{\theta_t\}$ and the constrained trajectory $\{\thetad_t\}$. Viewed as a dynamical system, these equations give rise to a negative feedback loop, which prevents both the sharpness and the displacement in the unstable direction from diverging. These equations also allow us to predict the values of the sharpness and the loss throughout the gradient descent trajectory.

\subsection{The Four Stages of Edge of Stability: A Heuristic Derivation}\label{sec:four_stages}

The analysis in this section proceeds by a cubic Taylor expansion around a fixed reference point $\theta^\star := \thetad_0$.\footnote{Beginning in \Cref{sec:theory}, the reference points for our Taylor expansions change at every step to minimize errors. However, fixing the reference point in this section simplifies the analysis, better illustrates the negative feedback loop, and motivates the definition of the constrained trajectory.}
For notational simplicity, we will define the following quantities at $\theta^\star$:
\begin{alignat*}{5} 
    \nabla L &:= \nabla L(\theta^\star),&\qquad \nabla^2 L &:= \nabla^2 L(\theta^\star),&\qquad  u &:= u(\theta^\star) \\
    \nabla S &:= \nabla S(\theta^\star),&\qquad \alpha &:= \alpha(\theta^\star),&\qquad \beta &:= \norm{\nabla S}^2,
\end{alignat*}
where $\alpha = -\nabla L \cdot \nabla S > 0$ is the progressive sharpening coefficient at $\theta^\star$. For simplicity, in \Cref{sec:heuristic} we assume that $\nabla S \perp u$ and $\nabla L,\nabla S \in \mathrm{ker}(\nabla^2 L)$, and ignore higher order error terms.\footnote{We give an explicit example of a loss function satisfying these assumptions in \Cref{sec:toy_model}.} Our main argument in \Cref{sec:theory} does not require these assumptions and tracks all error terms explicitly.

We want to track the movement in the unstable direction $u$ and the direction of changing sharpness $\nabla S$, and thus define 
\begin{align*}
    x_t := u \cdot (\theta_t - \theta^\star) \qand y_t := \nabla S \cdot (\theta_t - \theta^\star).
\end{align*}
Note that $y_t$ is approximately equal to the change in sharpness from $\theta^\star$ to $\theta_t$, since Taylor expanding the sharpness yields
\begin{align*}
    S(\theta_t) \approx S(\theta^\star) + \nabla S \cdot (\theta_t - \theta^\star) = 2/\eta + y_t.
\end{align*}

The mechanism for edge of stability can be described in 4 stages (see \Cref{fig:stages}):
\paragraph{Stage 1: Progressive Sharpening} While $x,y$ are small, $\nabla L(\theta_t) \approx \nabla L$. In addition, because $\nabla L \cdot \nabla S < 0$, gradient descent naturally increases the sharpness at every step. In particular,
\begin{align*}
    y_{t+1} - y_t = \nabla S \cdot(\theta_{t+1} - \theta_t) \approx - \eta \nabla L \cdot \nabla S = \eta \alpha.
\end{align*}
The sharpness therefore increases linearly with rate $\eta \alpha$.
\paragraph{Stage 2: Blowup} As $x_t$ measures the deviation from $\theta^\star$ in the $u$ direction, the dynamics of $x_t$ can be modeled by gradient descent on a quadratic with sharpness $S(\theta_t) \approx 2/\eta + y_t$. In particular, the rule for gradient descent on a quadratic gives\footnote{A rigorous derivation of this update in terms of $S(\theta_t)$ instad of $S(\theta^\star)$ requires a third-order Taylor expansion around $\theta^\star$; see \Cref{sec:proofs} for more details.}
\begin{align*}
    x_{t+1} = x_t - \eta u \cdot \nabla L(\theta_t) \approx x_t - \eta S(\theta_t) x_t \approx x_t - \eta \qty[2/\eta + y_t] x_t = -(1 + \eta y_t) x_t.
\end{align*}
When the sharpness exceeds $2/\eta$, i.e. when $y_t > 0$, $\abs{x_t}$ begins to grow exponentially indicating divergence in the top eigenvector $u$ direction.

\paragraph{Stage 3: Self-Stabilization} Once the movement in the $u$ direction is sufficiently large, the loss is no longer locally quadratic. Understanding the dynamics necessitates a third order Taylor expansion. The missing cubic term in the Taylor expansion of $\nabla L(\theta_t)$ is
\begin{align*}
    \frac12\nabla^3 L(\theta - \theta^\star,\theta - \theta^\star) \approx \nabla^3 L(u,u)\frac{x_t^2}{2} = \nabla S\frac{x_t^2}{2}
\end{align*}
by \Cref{lem:nabla_eig}. This biases the optimization trajectory in the $-\nabla S$ direction, which decreases sharpness.
Recalling $\beta = \norm{\nabla S}^2$, the new update for $y$ becomes:
\begin{align*}
    y_{t+1} -y_t &= \eta \alpha + \nabla S \cdot \qty(-\eta\nabla^3L(u, u)\frac{x_t^2}{2}) = \eta\qty(\alpha - \beta \frac{x_t^2}{2})
\end{align*}
Therefore once $x_t > \sqrt{2\alpha/\beta}$, the sharpness begins to decrease and continues to do so until the sharpness goes below $2/\eta$ and the dynamics return to stability.

\paragraph{Stage 4: Return to Stability} At this point $\abs{x_t}$ is still large from stages 1 and 2. However, the self-stabilization of stage 3 eventually drives the sharpness below $2/\eta$ so that $y_t < 0$. Because the rule for gradient descent on a quadratic with sharpness $S(\theta_t) = 2/\eta + y_t$ is still
\begin{align*}
    x_{t+1} \approx -(1 + \eta y_t) x_t,
\end{align*}
$\abs{x_t}$ begins to shrink exponentially and the process returns to stage 1.

Combining the update for $x_t, y_t$ in all four stages, we obtain the following simplified dynamics:
\begin{align}\label{eq:simple_update}
    x_{t+1} \approx -(1+\eta y_t)x_t \qand y_{t+1} \approx y_t + \eta\qty(\alpha - \beta \frac{x_t^2}{2})
\end{align}
where we recall $\alpha = -\nabla L \cdot \nabla S$ is the progressive sharpening coefficient and $\beta = \norm{\nabla S}^2$.

\subsection{Analyzing the simplified dynamics}\label{sec:ode_potential}
\begin{figure}
    \centering
    \includegraphics[width=\textwidth]{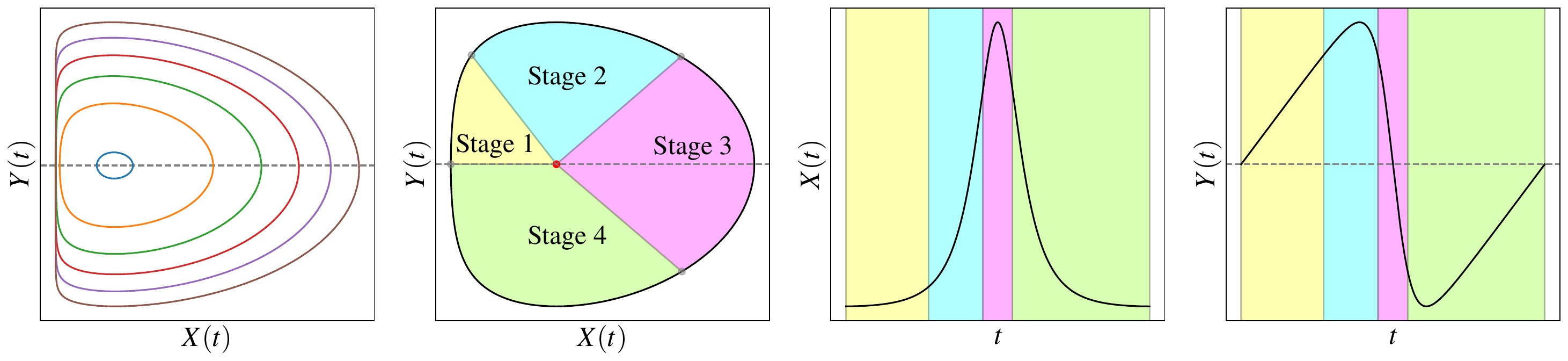}
    \caption{\textbf{The effect of $\mathbf{X(0)}$ (left):} We plot the evolution of the ODE in \cref{eq:bean_ode} with $\alpha = \beta = 1$ for varying $X(0)$. Observe that smaller $X(0)$'s correspond to larger curves. \textbf{The four stages of edge of stability (right):} We show how the four stages of edge of stability described in \Cref{sec:four_stages} and \Cref{fig:stages} correspond to different parts of the curve generated by the ODE in \cref{eq:bean_ode}.}
    \label{fig:bean_plots}
\end{figure}
We now analyze the dynamics in \cref{eq:simple_update}. First, note that $x_t$ changes sign at every iteration, and that, $x_{t+1} \approx -x_t$ due to the instability in the $u$ direction. While \cref{eq:simple_update} cannot be directly modeled by an ODE due to these rapid oscillations, we can instead model $\abs{x_t}, y_t$, whose update is controlled by $\eta$.
As a consequence, we can couple the dynamics of $\abs{x_t},y_t$ to the following ODE $X(t),Y(t)$:
\begin{align}\label{eq:bean_ode}
    X'(t) = X(t) Y(t) \qand Y'(t) = \alpha - \beta \frac{X(t)^2}{2}.
\end{align}
This system has the unique fixed point $(X,Y) = (\delta,0)$ where $\delta := \sqrt{2\alpha/\beta}$. We also note that this ODE can be written as a Lotka-Volterra predator-prey model after a change of variables, which is a classical example of a negative feedback loop. In particular, the following quantity is conserved:
\begin{lemma}
    Let $h(z) := z - \log z - 1$. Then the quantity 
    \begin{align*}
        g(X(t),Y(t)) := h\qty(\frac{\beta X(t)^2}{2\alpha}) + \frac{Y(t)^2}{\alpha}
    \end{align*} is conserved.
\end{lemma}
\begin{proof} 
\begin{align*}
    \frac{d}{dt} g(X(t),Y(t))
        = \frac{\beta X(t)^2 Y(t)}{\alpha} - 2 Y(t) + \frac{2}{\alpha} Y(t) \qty[\alpha - \beta \frac{X(t)^2}{2}] = 0. \qquad\qedhere
\end{align*}
\end{proof}
As a result we can use the conservation of $g$ to explicitly bound the size of the trajectory:
\begin{corollary}
    For all $t$, $X(0) \le X(t) \lesssim \delta\sqrt{\log(\delta/X(0))}$ and $\abs{Y(t)} \lesssim \sqrt{\alpha \log(\delta/X(0))}$.
\end{corollary}
The fluctuations in sharpness are $\tilde O(\sqrt{\alpha})$, while the fluctuations in the unstable direction are $\tilde O(\delta)$. Moreover, the \emph{normalized} displacement in the $\nabla S$ direction, i.e. $\frac{\nabla S}{\norm{\nabla S}} \cdot (\theta - \thetas)$ is also bounded by $\tilde O(\delta)$, so the entire process remains bounded by $\tilde O(\delta)$. Note that the fluctuations increase as the progressive sharpening constant $\alpha$ grows, and decrease as the self-stabilization force $\beta$ grows.

\subsection[Relationship with the constrained trajectory]{Relationship with the constrained trajectory $\thetad_t$}
\Cref{eq:simple_update} completely determines the displacement $\theta_t - \theta^\star$ in the $u,\nabla S$ directions and \Cref{sec:ode_potential} shows that these dynamics remain bounded by $\tilde O(\delta)$ where $\delta = \sqrt{2\alpha/\beta}$. However, progress is still made in all other directions. Indeed, $\theta_t$ evolves in these orthogonal directions by a simple projected gradient update: $-\eta P_{u,\nabla S}^\perp \nabla L$ at every step where $P_{u,\nabla S}^\perp$ is the projection onto this orthogonal subspace. This can be interpreted as first taking a gradient step of $-\eta \nabla L$ and then projecting out the $\nabla S$ direction to ensure the sharpness does not change. \Cref{lem:dagger_step}, given in the Appendix, shows that this is precisely the update for $\thetad_t$ (\cref{eq:constrained_update}) up to higher order terms. The preceding derivation thus implies that $\|\theta_t - \thetad_t\| \le \tilde O(\delta)$ and that this $\tilde O(\delta)$ error term is controlled by the self-stabilizing dynamics in \cref{eq:simple_update}.

\section{The Predicted Dynamics and Theoretical Results}\label{sec:theory}
We now present the equations governing edge of stability for general loss functions.
\subsection{Notation}
Our general approach Taylor expands the gradient of each iterate $\theta_t$ around the corresponding iterate $\thetad_t$ of the constrained trajectory. We define the following Taylor expansion quantities at $\thetad_t$:
\begin{definition}[Taylor Expansion Quantities at $\thetad_t$]~
\begin{gather*}
    \nabla L_t := \nabla L(\thetad_t) \qc \nabla^2 L_t := \nabla^2 L(\thetad_t) \qc \nabla^3 L_t := \nabla^3 L(\thetad_t) \\ \nabla S_t := \nabla S(\thetad_t) \qc\qquad u_t := u(\thetad_t).
\end{gather*}
Furthermore, for any vector-valued function $v(\theta)$, we define $v_t^\perp := P_{u_t}^\perp v(\thetad_t)$ where $P_{u_t}^\perp$ is the projection onto the orthogonal complement of $u_t$.
\end{definition}

We also define the following quantities which govern the dynamics near $\thetas_t$.
\begin{definition}\label{def:alpha_beta_epsilon} Define
\begin{align*}
    \alpha_t := -\nabla L_t \cdot \nabla S_t \qc \beta_t := \norm{\nabla S_t^\perp}^2 \qc \delta_t := \sqrt{\frac{2\alpha_t}{\beta_t}} \qand \delta := \sup_t \delta_t.
\end{align*}
Furthermore, we define 
\begin{align*}
    \beta_{s \to t} := \nabla S_{t+1}^\perp\qty[\prod_{k = t}^{s+1} (I - \eta \nabla^2L_k)P_{u_k}^\perp] \nabla S_s^\perp.
\end{align*}
\end{definition}
Recall that $\alpha_t$ is the progressive sharpening force, $\beta_t$ is the strength of the stabilization force, and $\delta_t$ controls the size of the deviations from $\thetad_t$ and was the fixed point in the $x$ direction in \Cref{sec:ode_potential}. In addition, $\beta_{s \to t}$ admits a simple interpretation: it measures the change in sharpness between times $s$ and $t+1$ after a displacement of $\nabla S_s^\perp$ at time $s$. Unlike in \Cref{sec:heuristic} where $\nabla S^\perp$ is always perpendicular to the Hessian, the Hessian causes this displacement to change at every step. In particular, at time $k$, it multiplies the displacement by $(I - \eta \nabla^2 L_k)$. The displacement after $t$ steps is therefore
\begin{align*}
    \mathtt{displacement} \approx \qty[\prod_{k=t}^{s+1} (I - \eta \nabla^2 L_k) P_{u_k}^\perp] \nabla S_s^\perp.
\end{align*}
As the change in sharpness is approximately equal to $\nabla S_{t+1} \cdot \mathtt{displacement}$, the change in sharpness between times $s$ and $t+1$ is approximately captured by $\beta_{s \to t}$.

When $\nabla S$ is constant and orthogonal to the Hessian, as in \Cref{sec:heuristic}, this change is $\beta = \norm{\nabla S^\perp}^2$ because the change in sharpness is approximately $\nabla S \cdot \mathtt{displacement}$ and the displacement is $\nabla S^\perp$. Therefore in the setting of \Cref{sec:heuristic}, $\beta_{s \to t}$ is constant.

\subsection{The equations governing edge of stability}
We now introduce the equations governing edge of stability. We track the following quantities:
\begin{definition}\label{def:vxy}
    Define $v_t := \theta_t - \thetad_t$, $x_t := u_t \cdot v_t$, $y_t := \nabla S_t^\perp\cdot v_t$.
\end{definition}
Our predicted dynamics directly predict the displacement $v_t$ and the full definition is deferred to \Cref{sec: define predicted dynamics}. However, they have a relatively simple form in the $u_t,\nabla S_t^\perp$ directions that only depend on the remaining $v$ directions through the scalar quantities $\beta_{s \to t}$. 

\begin{lemma}[Predicted Dynamics for $x,y$]\label{lemma:predicted_x_y}
    Let $\vs_t$ denote our predicted dynamics (defined in \Cref{sec: define predicted dynamics}). Letting $\xs_t = u_t \cdot \vs_t$ and $\ys_t = \nabla S_t^\perp \cdot \vs_t$, we have
    \begin{align}\label{eq:predicted_x_y_only}
    \xs_{t+1} = - (1 + \eta \ys_t)\xs_t \qand \ys_{t+1} = \eta\sum_{s = 0}^t \beta_{s \to t}\qty[\frac{\delta_s^2 - {\xs_s}^2}{2}].
\end{align}
\end{lemma}
As we will see in Theorem \ref{thm:coupling}, the $x,y$ directions alone suffice to determine the loss and sharpness values and, in fact, \textit{the EOS dynamics can be fully captured by this 2d dynamical system with time-dependent coefficients}.

 Note that when $\beta_{s \to t}$ are constant, our update reduces to the simple case discussed in \Cref{sec:heuristic}, which we analyze fully. When $x_t$ is large, \cref{eq:predicted_x_y_only} demonstrates that there is a self-stabilization force which acts to decrease $y_t$; however, unlike in \Cref{sec:heuristic}, the strength of this force changes with $t$.

\subsection{Coupling Theorem}
We now show that, under a mild set of assumptions which we verify to hold empirically in \Cref{sec:verify_assumptions}, the true dynamics are accurately governed by the predicted dynamics. This lets us use the predicted dynamics to predict the loss, sharpness, and the distance to the constrained trajectory $\thetad_t$.

Our errors depend on the unitless quantity $\epsilon$, which we verify is small in \Cref{sec:verify_assumptions}.
\begin{definition}
    Let $\epsilon_t := \eta \sqrt{\alpha_t}$ and $\epsilon := \sup_t \epsilon_t$.
\end{definition}
To control Taylor expansion errors, we require upper bounds on $\nabla^3 L$ and its Lipschitz constant:\footnote{For simplicity of exposition, we make these bounds on $\nabla^3 L$ globally, however our proof only requires them in a small neighborhood of the constrained trajectory $\theta^\dagger$.}
\begin{assumption}\label{assume:rho4}
    Let $\rho_3$, $\rho_4$ to be the minimum constants such that for all $\theta$, $\norm{\nabla^3 L(\theta)}_{op} \le \rho_3$ and $\nabla^3 L$ is $\rho_4$-Lipschitz with respect to $\norm{\cdot}_{op}$. Then we assume that $\rho_4 = O(\eta \rho_3^2)$.
\end{assumption}
Next, we require the following generalization of \Cref{assumption:progressive_sharpening}: 
\begin{assumption}\label{assume:prog_general}
    For all $t$,
\begin{align*}
    \frac{-\nabla L_t \cdot \nabla S_t}{\norm{\nabla L_t}\norm{\nabla S_t^\perp}} = \Theta(1) \qand \norm{\nabla S_t^\perp} = \Theta(\rho_3).
\end{align*}
\end{assumption}

Finally, we require a set of ``non-worst-case" assumptions, which are that the quantities $\nabla^2 L, \nabla^3 L,$ and $\lambda_{min}(\nabla^2 L)$ are nicely behaved in the directions orthogonal to $u_t$, which generalizes the eigengap assumption. We verify the assumptions on $\nabla^2 L$ and $\nabla^3 L$ empirically in \Cref{sec:verify_assumptions}.
\begin{assumption}\label{assume:non_worst} For all $t$ and $v,w \perp u_t$,
\begin{align*}\frac{\norm{\nabla^3L_t(v, w)}}{\norm{\nabla^3 L_t}_{op}\norm{v}\norm{w}} \qc \frac{|\nabla^2 L_t( \vs^\perp_t,\vs^\perp_t)|}{\norm{\nabla^2 L_t}\norm{\vs^\perp_t}^2} \qand \frac{\abs{\lambda_{min}(\nabla^2 L_t)}}{\|\nabla^2 L_t\|_2} \le O(\epsilon).
\end{align*}
\end{assumption}

With these assumptions in place, we can state our main theorem which guarantees $\xs,\ys,\vs$ predict the loss, sharpness, and deviation from the constrained trajectory up to higher order terms:
\begin{theorem}\label{thm:coupling}
    Let $\mathscr{T} := O(\epsilon^{-1})$ and assume that $\min_{t \le \mathscr{T}} \abs{\xs_t}\ge c_1\delta.$ Then for any $t \le \mathscr{T}$, we have
    \begin{align*}
        L(\theta_t) &= L(\thetad_t) + \xs_t^2/\eta +  O\qty(\epsilon \delta^2/\eta) \tag{Loss} \\
        S(\theta_t) &= 2/\eta + \ys_t + \qty(S_t \cdot u_t) \xs_t + O\qty(\epsilon^2/\eta) \tag{Sharpness} \\
        \theta_t &= \thetad_t + \vs_t + O\qty(\epsilon\delta) \tag{Deviation from $\theta^\dagger$}
    \end{align*}
\end{theorem}

\Cref{thm:coupling} says that up to higher order terms, the predicted dynamics $\vs_t$ capture the deviation from the constrained trajectory and allow us to predict the sharpness and loss of the current iterate. In particular, the GD trajectory $\{\theta_t\}_t$ should be though of as following the constrained PGD trajectory $\{\theta^\dagger\}_t$ plus a rapidly oscillating process whose dynamics are given by the predicted dynamics \cref{eq:predicted full}.

The sharpness is controlled by the slowly evolving quantity $\ys_t$ and the period-2 oscillations of $(\nabla S \cdot U) \xs_t$. This combination of gradual and rapid periodic behavior was observed by \citet{cohen2021eos} and appears in our experiments. \Cref{thm:coupling} also shows that the loss at $\theta_t$ spikes whenever $\xs_t$ is large. On the other hand, when $\xs_t$ is small, $L(\theta_t)$ approaches the loss of the constrained trajectory.

\section{Experiments}\label{sec:experiments}

\begin{figure}
    \centering
    \includegraphics[width=0.9\textwidth]{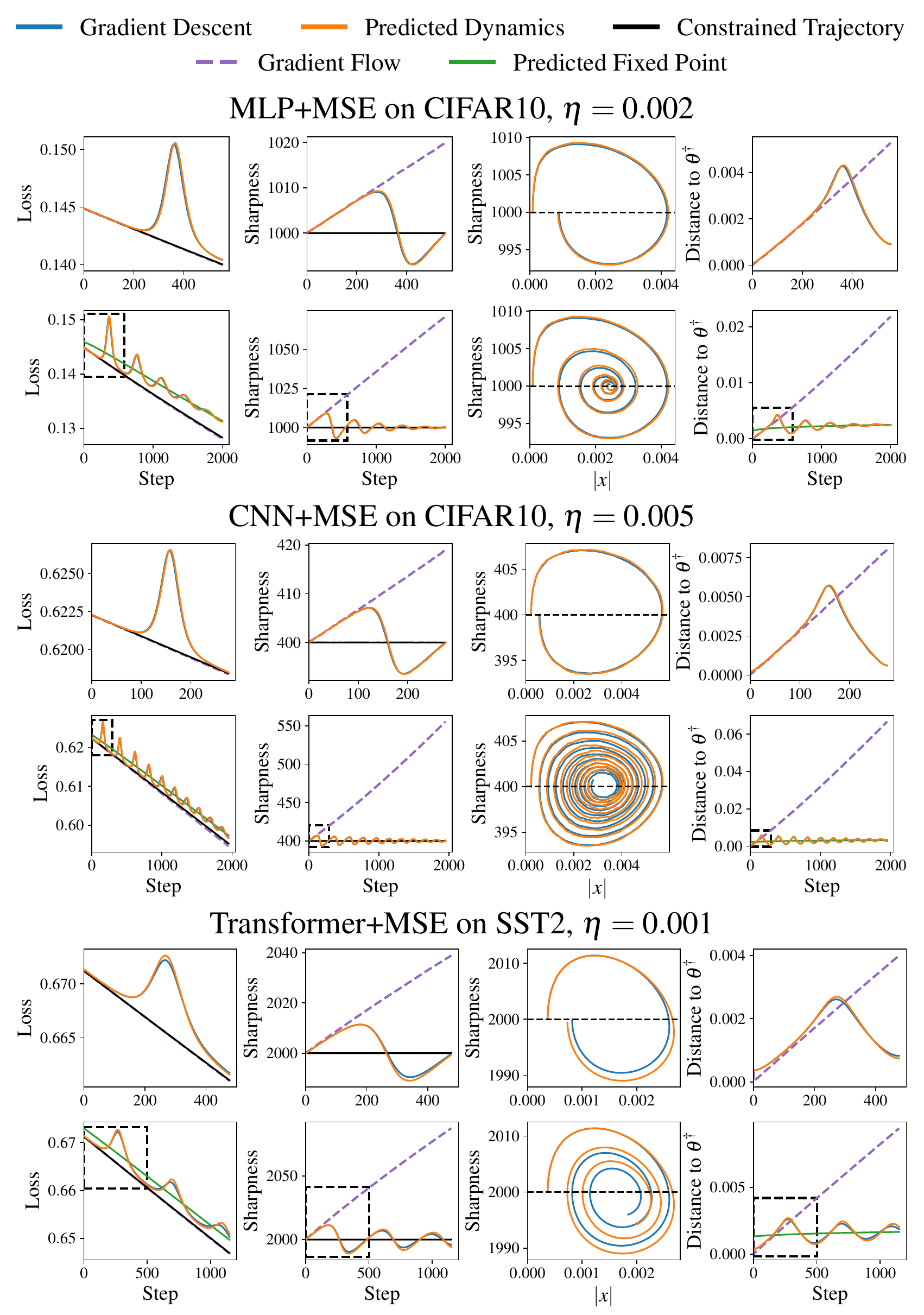}
    \caption{We empirically demonstrate that the predicted dynamics given by \cref{eq:predicted_x_y_only} track the true dynamics of gradient descent at the edge of stability. For each learning rate, the top row is a zoomed in version of the bottom row which isolates one cycle and is reflected by the dashed rectangle in the bottom row. Reported sharpnesses are two-step averages for visual clarity. For additional experimental details, see \Cref{sec:experiments} and \Cref{sec:experimental_details}.}
    \label{fig:loss_sharpness}
\end{figure}

We verify that the predicted dynamics defined in \cref{eq:predicted_x_y_only} accurately capture the dynamics of gradient descent at the edge of stability by replicating the experiments in \citep{cohen2021eos} and tracking the deviation of gradient descent from the constrained trajectory. In \Cref{fig:loss_sharpness}, we evaluate our theory on a 3-layer MLP and a 3-layer CNN trained with mean squared error (MSE) on a 5k subset of CIFAR10 and a 2-layer Transformer \citep{Vaswani2017AttentionIA} trained with MSE on SST2 \citet{socher-etal-2013-recursive}. We provide additional experiments varying the learning rate and loss function in \Cref{sec:additional_experiments}, which use the generalized predicted dynamics described in \Cref{sec:generalized_discussion}. For additional details, see \Cref{sec:experimental_details}.

\Cref{fig:loss_sharpness} confirms that the predicted dynamics \cref{eq:predicted_x_y_only} accurately predict the loss, sharpness, and distance from the constrained trajectory. In addition, while the gradient flow trajectory diverges from the gradient descent trajectory at a linear rate, the gradient descent trajectory and the constrained trajectories remain close \emph{throughout training}. In particular, the dynamics converge to the fixed point $(\abs{x_t},y_t) = (\delta_t,0)$ described in \Cref{sec:ode_potential} and $\|\theta_t - \theta^\dagger_t\| \to \delta_t$. This confirms our claim that gradient descent implicitly follows the constrained trajectory~\cref{eq:constrained_update}.

In \Cref{sec:theory}, various assumptions on the model were made to obtain the edge of stability behavior. In \Cref{sec:verify_assumptions}, we numerically verify these assumptions to ensure the validity of our theory.

\section{Discussion}\label{sec:discussion}

\subsection{Takeaways from the Predicted Dynamics}

Recall that the predicted dynamics $\vs_t,\xs_t,\ys_t$ describe the deviation of the GD trajectory $\{\theta_t\}_t$ from the PGD constrained trajectory $\{\thetad_t\}_t$. These dynamics enable many interesting observations about the EOS dynamics. First, the loss and sharpness only depend on the quantities $(\xs_t, \ys_t)$, which are governed by the 2D dynamical system with time-dependent coefficients \cref{eq:predicted_x_y_only}. When $\alpha_t, \beta_{s \to t}$ are constant, we showed that this system cycles and has a conserved potential. In general, understanding the edge of stability dynamics only requires analyzing the 2D system \cref{eq:predicted_x_y_only}, which is generally well behaved (\Cref{fig:loss_sharpness}).

In the limit, we expect $\xs_t, \ys_t$ to approach $(\pm \delta_t, 0)$, the fixed point of the system \cref{eq:predicted_x_y_only}.
In fact, \Cref{fig:loss_sharpness} shows that after a few cycles, $(\xs_t, \ys_t)$ indeed converges to this fixed point. We are able to accurately predict its location, as well as the loss increase from the constrained trajectory due to $\xs_t \neq 0$.
\subsection{Generalized Predicted Dynamics}\label{sec:generalized_discussion}
In order for our cubic Taylor expansions to track the true gradients, we need a bound on the fourth derivative of the loss (\Cref{assume:rho4}). This is usually sufficient to capture the dynamics at the edge of stability as demonstrated by \Cref{fig:loss_sharpness} and \Cref{sec:verify_assumptions}. However, this condition was violated in some of our experiments, especially when using logistic loss. To overcome this challenge, we developed a generalized form of the predicted dynamics whose definition we defer to \Cref{sec:generalized_dynamics}. These generalized predictions are qualitatively similar to those given by the predicted dynamics in \Cref{sec:theory}; however, they precisely track the dynamics of gradient descent in a wider range of settings. See \Cref{sec:additional_experiments} for empirical verification of the generalized predicted dynamics.

\subsection{Implications for Neural Network Training}
\paragraph{Non-Monotonic Loss Decrease} A central phenomenon at edge of stability is that despite non-monotonic fluctuations of the loss, the loss still decreases over long time scales. Our theory provides a clear explanation for this decrease. We show that the gradient descent trajectory remains close to the constrained trajectory (\Cref{sec:heuristic,sec:theory}). Since the constrained trajectory is \emph{stable}, it satisfies a descent lemma (\Cref{lem:dagger_descent}), and has monotonically decreasing loss. Over short time periods, the loss is dominated by the rapid fluctuations of $x_t$ described in \Cref{sec:heuristic}. Over longer time periods, the loss decrease of the constrained trajectory due to the descent lemma overpowers the bounded fluctuations of $x_t$, leading to an overall loss decrease. This is reflected in our experiments in \Cref{sec:experiments}.

\paragraph{Generalization \& the Role of Large Learning Rates} Prior work has shown that in neural networks, both decreasing sharpness of the learned solution~\citep{keskar2017, dziugaite2017, neyshabur2017, Jiang2020Fantastic} and increasing the learning rate~\citep{smith2018dont, li2019large, Lewkowycz2020} are correlated with better generalization. Our analysis shows that gradient descent implicitly constrains the sharpness to stay near $2/\eta$, which suggests larger learning may improve generalization by reducing the sharpness. In \Cref{fig:training_speed} we confirm that in a standard setting, full-batch gradient descent generalizes better with large learning rates.

\paragraph{Training Speed} Additional experiments in \citep[Appendix F]{cohen2021eos} show that, despite the instability in the training process, larger learning rates lead to faster convergence. 
This phenomenon is explained by our analysis. Gradient descent is coupled to the constrained trajectory which minimizes the loss while constraining movement in the $u_t,\nabla S_t^\perp$ directions. Since only two directions are ``off limits,'' the constrained trajectory can still move quickly in the orthogonal directions, using the large learning rate to accelerate convergence. We demonstrate this empirically in \Cref{fig:training_speed}.

\paragraph{Connection to Sharpness Aware Minimization (SAM)} \citet{foret2021sharpnessaware} introduced the sharpness-aware minimization (SAM) algorithm, which aims to control sharpness by solving the optimization problem $\min_\theta\max_{\|\delta\| \le \epsilon} L(\theta + \delta)$. This is roughly equivalent to minimizing $S(\theta)$ over all global minimizers, and thus SAM tries to explicitly minimize the sharpness. Our analysis shows that gradient descent \emph{implicitly} minimizes the sharpness, and for a fixed $\eta$ looks to minimize $L(\theta)$ subject to $S(\theta) = 2/\eta$. 

\paragraph{Connection to Warmup} \citet{gilmer2022} demonstrated that \emph{learning rate warmup}, which consists of gradually increasing the learning rate, empirically leads to being able to train with a larger learning rate. The self-stabilization property of gradient descent provides a plausible explanation for this phenomenon. If too large of an initial learning rate $\eta_0$ is chosen (so that $S(\theta_0)$ is much greater than $2/\eta_0$), then the iterates may diverge before self stabilization can decrease the sharpness to $2/\eta_0$. On the other hand, if the learning rate is chosen that $S(\theta_0)$ is only slightly greater than $2/\eta_0$, self-stabilization will decrease the sharpness to $2/\eta_0$. Repeatedly increasing the learning rate slightly could then lead to small decreases in sharpness without the iterates diverging, thus allowing training to proceed with a large learning rate.

\paragraph{Connection to Weight Decay and Sharpness Reduction.} \citet{Lyu2022Normalization} proved that when the loss function is scale-invariant, gradient descent with weight decay and sufficiently small learning rate converges leads to reduction of the \emph{normalized} sharpness $S(\theta/\|\theta\|)$. In fact, the mechanism behind the sharpness reduction is exactly the self-stabilization force described in this paper restricted to the setting in \citep{Lyu2022Normalization}. In section \Cref{sec:scale_invariant} we present a heuristic derivation of this equivalence.

\section{Future Work}\label{sec:future_work}

\subsection{Towards A Complete Theory of Optimization}

Our main result \Cref{thm:coupling} gives sufficient and verifiable conditions under which the GD trajectory $\{\theta_t\}_t$ and the PGD constrained trajectory $\{\thetad_t\}_t$ can be coupled for $O(\epsilon^{-1})$ steps. However, these conditions are not strictly necessary and our local coupling result is not sufficient to prove global convergence to a stationary point of \cref{eq:constrained_update}. These suggest two important directions for future work: Can we precisely characterize the prerequisites on the loss function and learning rate for self-stabilization to occur? Can we couple for longer periods of time or repeat this local coupling result to prove convergence to KKT points of \cref{eq:constrained_update}?

\subsection{Multiple Unstable Eigenvalues} Our work focuses on explaining edge of stability in the presence of a single unstable eigenvalue (\Cref{assumption:eigval_gap}). However, \citet{cohen2021eos} observed that progressive sharpening appears to apply to \emph{all} eigenvalues, even after the largest eigenvalue has become unstable. As a result, all of the top eigenvalues will successively enter edge of stability (see \Cref{fig:multiple_eval}). In particular, \Cref{fig:multiple_eval} shows that the dynamics are fairly well behaved in the period when only a single eigenvalue is unstable, yet appear to be significantly more chaotic when multiple eigenvalues are unstable.

\begin{figure}[h!]
    \centering
    \includegraphics[width=0.9\textwidth]{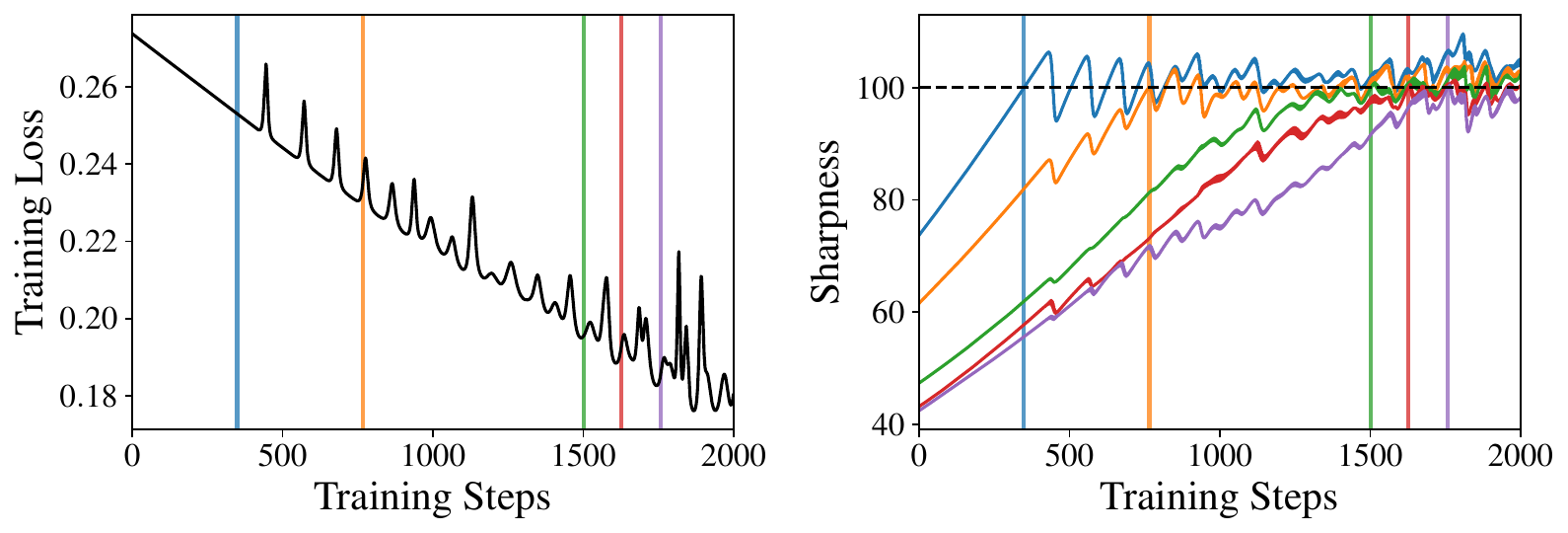}
    \caption{Edge of stability with multiple unstable eigenvalues. Each vertical line is the time at which the corresponding eigenvalue of the same color becomes unstable.}
    \label{fig:multiple_eval}
\end{figure} 

One technical challenge with dealing with multiple eigenvalues is that, when the top eigenvalue is not unique, the sharpness is no longer differentiable and it is unclear how to generalize our analysis. However, one might expect that gradient descent can still be coupled to projected gradient descent under the non-differentiable constraint $S(\thetad_T) \le 2/\eta$. When there are $k$ unstable eigenvalues, with corresponding eigenvectors $u_t^1, \ldots, u_t^k$, the constrained update is roughly equivalent to projecting out the subspace $\mathrm{span}\{\nabla^3L_t(u_t^i, u_t^j) : i,j \in [k] \}$ from the gradient update $-\eta\nabla L_t$. Demonstrating self-stabilization thus requires analyzing the dynamics in the subspace $\mathrm{span}\qty(\{u_t^i~:~i \in [k]\} \cup \{\nabla^3L_t(u_t^i, u_t^j) : i,j \in [k] \})$.
We leave investigating the dynamics of multiple unstable eigenvalues for future work.

\subsection{The Mystery of Progressive Sharpening}

Our analysis directly assumed the existence of progressive sharpening (\Cref{assumption:progressive_sharpening}), and focused on explaining the edge of stability dynamics using this assumption. However, this leaves open the question of \emph{why} neural networks exhibit progressive sharpening, which is an important question for future work. Partial progress towards understanding the mechanism behind progressive sharpening in neural networks has been made in the concurrent works~\citet{Li2022EoS,xingyu2022eosminimalist,agarwal2022quadraticeos}.

\subsection{Connections to Stochasticity}

Our analysis focused on understanding the edge of stability dynamics for gradient descent. However, phenomena similar to the edge of stability have also been observed for SGD \citep{jastrzebski2019on,Jastrzebski2020The,cohen2022adaptive_eos}. While these phenomena do not exhibit as simple of a characterization as the full batch gradient descent dynamics do, understanding the optimization dynamics of neural networks used in practice requires understanding the connections between edge of stabity and SGD. Important questions include what the correct notion of ``stability'' is for SGD and what form the self-stabilization force takes.

One possible hypothesis for how self-stabilization occurs in SGD is as a byproduct of the implicit regularization of SGD described in \citet{BlancGVV20, damian2021label, li2022what}. These works show that the stochasticity in SGD has the effect of decreasing a quantity related to the trace of the Hessian, rather than directly constraining the operator norm of the Hessian (i.e the sharpness).

Furthermore, \citet{damian2021label} showed that the implicit bias of label noise SGD is proportional to $-\log(1-\frac{\eta}{2} \nabla^2 L)$ which blows up as $S(\theta) \to 2/\eta$. This regularizer therefore heuristically acts as a log-barrier which enforces $S(\theta) \le 2/\eta$, rather than as a hard constraint. The precise ``break-even'' point \citep{Jastrzebski2020The} could then be approximated by the point at which this regularization force balances with progressive sharpening. It is an interesting direction to better understand the interactions between the self-stabilization mechanism described in this paper and the implicit regularization effects of SGD described in in \citet{BlancGVV20, damian2021label, li2022what}.

\section*{Acknowledgements}
AD acknowledges support from a NSF Graduate Research Fellowship. EN acknowledges support from a National Defense Science \& Engineering Graduate Fellowship, and NSF grants CIF-1907661 and DMS-2014279. JDL, AD, EN acknowledge support of  the Sloan Research Fellowship, NSF CCF 2002272, NSF IIS 2107304, NSF CIF 2212262, and NSF-CAREER under award \#2144994.

The authors would like to thank Jeremy Cohen, Kaifeng Lyu, and Lei Chen for helpful discussions throughout the course of this project. We would especially like to thank Jeremy Cohen for suggesting the term ``self-stabilization'' to describe the negative feedback loop derived in this paper.

\bibliography{ref}

\clearpage

\appendix

\renewcommand{\contentsname}{Table of Contents}
\tableofcontents

\section{Notation}
We denote by $\nabla^k L(\theta)$ the $k$-th order derivative of the loss $L$ at $\theta$. Note that $\nabla^k L(\theta)$ is a symmetric $k$-tensor in $(\mathbb{R}^{d})^{\otimes k}$ when $\theta \in \mathbb{R}^d$.

For a symmetric $k$-tensor $T$, and vectors $u_1,\ldots,u_j \in \mathbb{R}^d$ we will use $T(u_1,\ldots,u_j)$ to denote the tensor contraction of $T$ with $u_1,\ldots,u_j$, i.e.
\begin{align*}
    \qty[T(u_1,\ldots,u_k)]_{i_1,\ldots,i_{k-j}} := T_{i_1,\ldots,i_k} (u_1)_{i_{k-j+1}} \cdots (u_j)_{i_k}.
\end{align*}

We use $P_{u_1,\ldots,u_k}$ to denote the orthogonal projection onto $\mathrm{span}(u_1,\ldots,u_k)$ and $P_{u_1,\ldots,u_k}^\perp$ is the projection onto the corresponding orthogonal complement.

For matrices $A_1,\ldots,A_k$, we define
\begin{align*}
    \prod_{k=1}^t A_k := A_1 \ldots A_t \qand \prod_{k=t}^1 A_k := A_t \ldots A_1.
\end{align*}

\section{A Toy Model for Self-Stabilization}\label{sec:toy_model}

For $\alpha,\beta > 0$, consider the function
\begin{align*}
    L(x,y,z) := \qty(\frac{2}{\eta} + \sqrt{\beta} y)\frac{x^2}{2} - \frac{\alpha}{\sqrt{\beta}} y - z
\end{align*}
initialized at the point $(x_0,0,0)$. Note that the constrained trajectory will follow $x^\dagger_t = 0$, $y^\dagger_t = 0$, $z^\dagger_t = -\eta t$ as it cannot decrease $y$ without increasing the sharpness past $2/\eta$. We therefore have:
\begin{align*}
    \nabla L_t = \qty[0,-\frac{\alpha}{\sqrt{\beta}},1],\ u_t = [1,0,0],\ S_t = 2/\eta + \sqrt{\beta} y,\ \nabla^2 L_t = S_t u_t u_t^t,\ \nabla S_t = \qty[0,\sqrt{\beta},0].
\end{align*}
Note that this satisfies all of the assumptions in \Cref{sec:heuristic} and it satisfies $\alpha = -\nabla L_t \cdot \nabla S_t$ and $\beta = \|\nabla S_t\|^2$. This process will then follow \cref{eq:bean_ode} in the $x,y$ directions while it tracks the constrained trajectory $\theta^\dagger_t$ moving linearly in the $-P_{u,\nabla S}^\perp \nabla L = [0,0,-1]$ direction.

\section{Definition of the Predicted Dynamics}\label{sec: define predicted dynamics}
Below, we present the full definition of the predicted dynamics:

\begin{definition}[Predicted Dynamics, full]\label{def:predicted_full}
Define $\vs_0 =v_0$, and let $\xs_t = \vs_t \cdot u_t, \ys_t = \nabla S^\perp \cdot \vs_t$. Then
\begin{align}\label{eq:predicted full}
    v^*_{t+1} = P_{u_{t+1}}^\perp(I - \eta \nabla^2 L_t) P_{u_t}^\perp v^*_t + \eta P_{u_{t+1}}^\perp\nabla S_t^\perp \qty[\frac{\delta_t^2 - {x^*_t}^2}{2}] -(1 + \eta y^*_t)x^*_t\cdot u_{t+1}
\end{align}
\end{definition}
For convenience, we will define the map $\step_t : \mathbb{R}^d \rightarrow \mathbb{R}^d$ as follows:
\begin{definition}
    Given a vector $v$ and a timestep $t$, define $\step_t(v)$ by
    \begin{align}
        P_{u_{t+1}}^\perp \step_t(v) &= P_{u_{t+1}}^\perp \qty[(I - \eta \nabla^2 L_t) P_{u_t}^\perp v + \eta \nabla S_t^\perp \qty[\frac{\delta_t^2 - x^2}{2}]] \\
        u_{t+1} \cdot \step_t(v) &= -(1 + \eta y)x.
    \end{align}
    where $x = u_t \cdot v$ and $y = \nabla S_t^\perp \cdot v$.
\end{definition}
It is easy to see that $\vs_{t+1} = \step_t(\vs_t)$.

\begin{proof}[Proof of \Cref{lemma:predicted_x_y}]
Defining $A_t = (I - \eta\nabla^2L_t)P_{u_t}^\perp$, we can unfold the recursion in $\cref{eq:predicted full}$ to obtain the following formula for $\vs_t$.
\begin{align}
    v^*_{t+1} = \eta\sum_{s = 0}^t P_{u_{t+1}}^\perp \qty[\prod_{k = t}^{s+1} A_k] \nabla S_s^\perp \qty[\frac{\delta_s^2 - {x^*_s}^2}{2}] - (1 + \eta y^*_t)x^*_t\cdot u_{t+1}.\label{eq:vt_star_formula}
\end{align}

It is then immediate to see that $\xs_t = \vs_t \cdot u_t, \ys_t = \nabla S_t^\perp \cdot \vs_t$ have the following simple update:
\begin{align*}
    x^*_{t+1} = - (1 + \eta y^*_t)x^*_t \qand y^*_{t+1} = \eta\sum_{s = 0}^t \beta_{s \to t}\qty[\frac{\delta_s^2 - {x^*_s}^2}{2}],
\end{align*}
where we recall that we have defined
\begin{align}
    \beta_{s \to t} := \nabla S_{t+1}^\perp\qty[\prod_{k = t}^{s+1} A_k] \nabla S_s^\perp.
\end{align}
\end{proof}

\section{Experimental Details}\label{sec:experimental_details}
\subsection{Architectures}
We evaluated our theory on four different architectures. The 3-layer MLP and CNN are exact copies of the MLP and CNN used in \citep{cohen2021eos}. The MLP has width $200$, the CNN has width $32$, and both are using the swish activation \citep{ramachandran2017swish}. We also evaluate on a ResNet18 with progressive widths $16,32,64,128$ and on a 2-layer Transformer with hidden dimension $64$ and two attention heads.

\subsection{Data}
We evaluated our theory on three primary tasks: CIFAR10 multi-class classification with both categorical MSE loss and cross-entropy loss, CIFAR10 binary classification (cats vs dogs) with binary MSE loss and logistic loss, and SST2 \citep{socher-etal-2013-recursive} with binary MSE loss and logistic loss.

\subsection{Experimental Setup}
For every experiment, we tracked the gradient descent dynamics until they reached instability and then began tracking the constrained trajectory, gradient descent, gradient flow, and both our predicted dynamics (\Cref{sec:theory}) and our generalized predicted dynamics (\Cref{sec:generalized_dynamics}). In addition, we tracked the various quantities on which we made assumptions for \Cref{sec:theory} in order to validate these assumptions. We also tracked the second eigenvalue of the Hessian at the constrained trajectory throughout training and stopped training once it reached $1.9/\eta$, to ensure the existence of a single unstable eigenvalue. Finally, as the edge of stability dynamics are very sensitive to small perturbation when $|x|$ is small (see \Cref{fig:bean_plots}), we switched to computing gradients with 64-bit precision after first reaching instability to avoid propagating floating point errors.

Eigenvalues were computed using the LOBPCG sparse eigenvalue solver in JAX \citep{jax2018github}. To compute the constrained trajectory, we computed a linearized approximation for $\proj_{\mathcal{M}}$ inspired by \Cref{lem:dagger_step} along with a Newton step in the $u_t$ direction to ensure that $\nabla L \cdot u = 0$. Each linearized approximation step required recomputing the sharpness and top eigenvector and each projection step then consisted of three linearized projection steps, for a total of three eigenvalue computations per projection step.

Our experiments were conducted in JAX \citep{jax2018github}, using \url{https://github.com/locuslab/edge-of-stability} as a reference for replicating the experimental setup used in \citep{cohen2021eos}. All experiments were conducted on two servers, each with 10 NVIDIA GPUs. Code is available at \url{https://github.com/adamian98/EOS}.

\section{Empirical Verification of the Assumptions}\label{sec:verify_assumptions}

For each of the experimental settings considered (MLP+MSE, CNN+MSE, CNN+Logistic, ResNet18+MSE, Transformer+MSE, Transformer+Logistic), we plot a number of quantities along the constrained trajectory to verify that the assumptions made in the main text hold. For each learning rate $\eta$ we have 8 plots tracking various quantities, which verify the assumptions as follows: \Cref{assumption:progressive_sharpening} is verified by the 1st plot, $\epsilon$ being small is verified by the 2nd plot, \Cref{assume:prog_general} is verified by the 3rd and 4th plots, \Cref{assume:rho4} is verified by the 5th plot, and \Cref{assume:non_worst} is verified by the last 3 plots. As described in the experimental setup, training is stopped once the second eigenvalue is $1.9/\eta$, so \Cref{assumption:eigval_gap} always holds with $c=1.9$ as well.
\begin{center}
\newpage
\includegraphics[width=\linewidth]{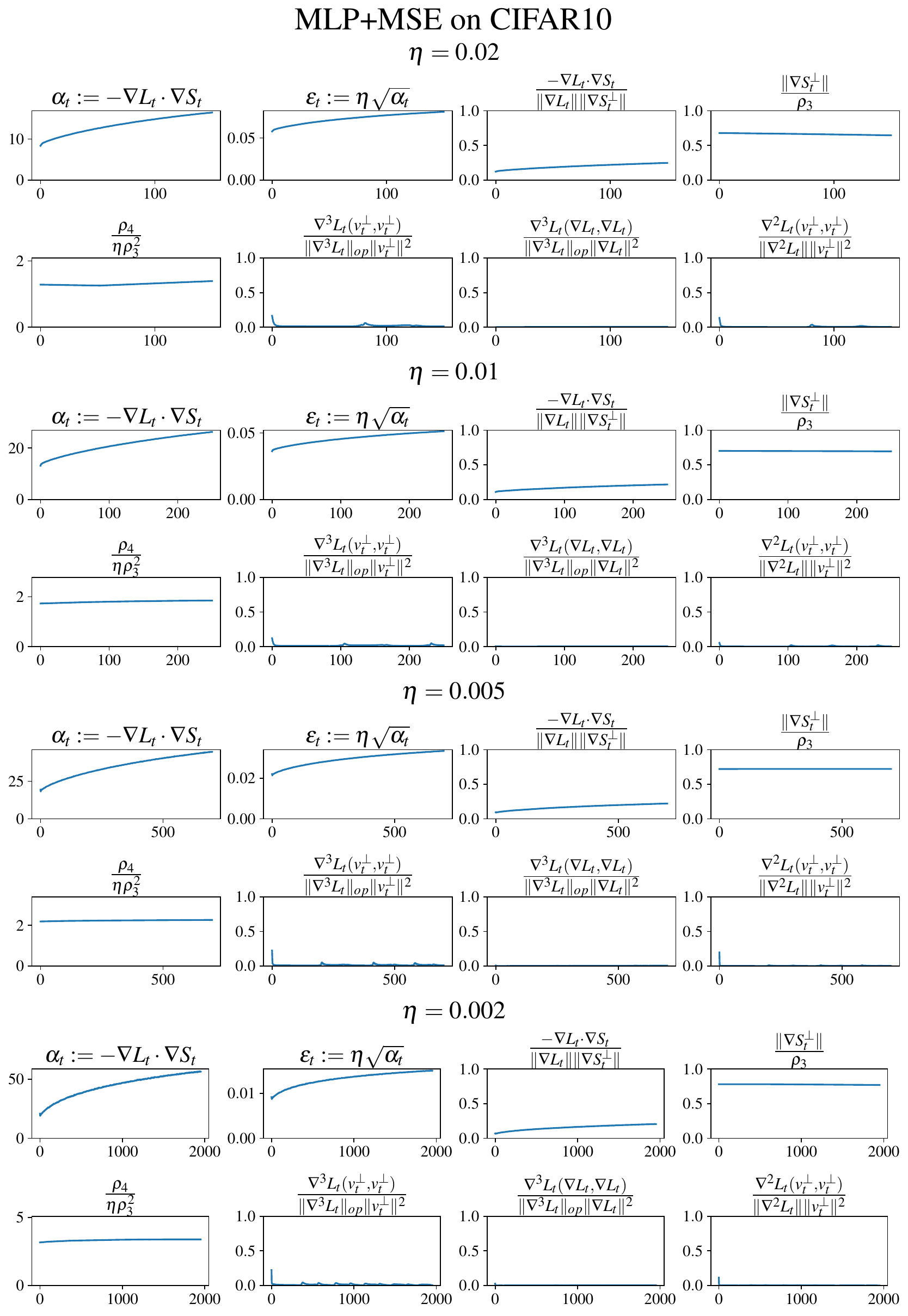}
\newpage
\includegraphics[width=\linewidth]{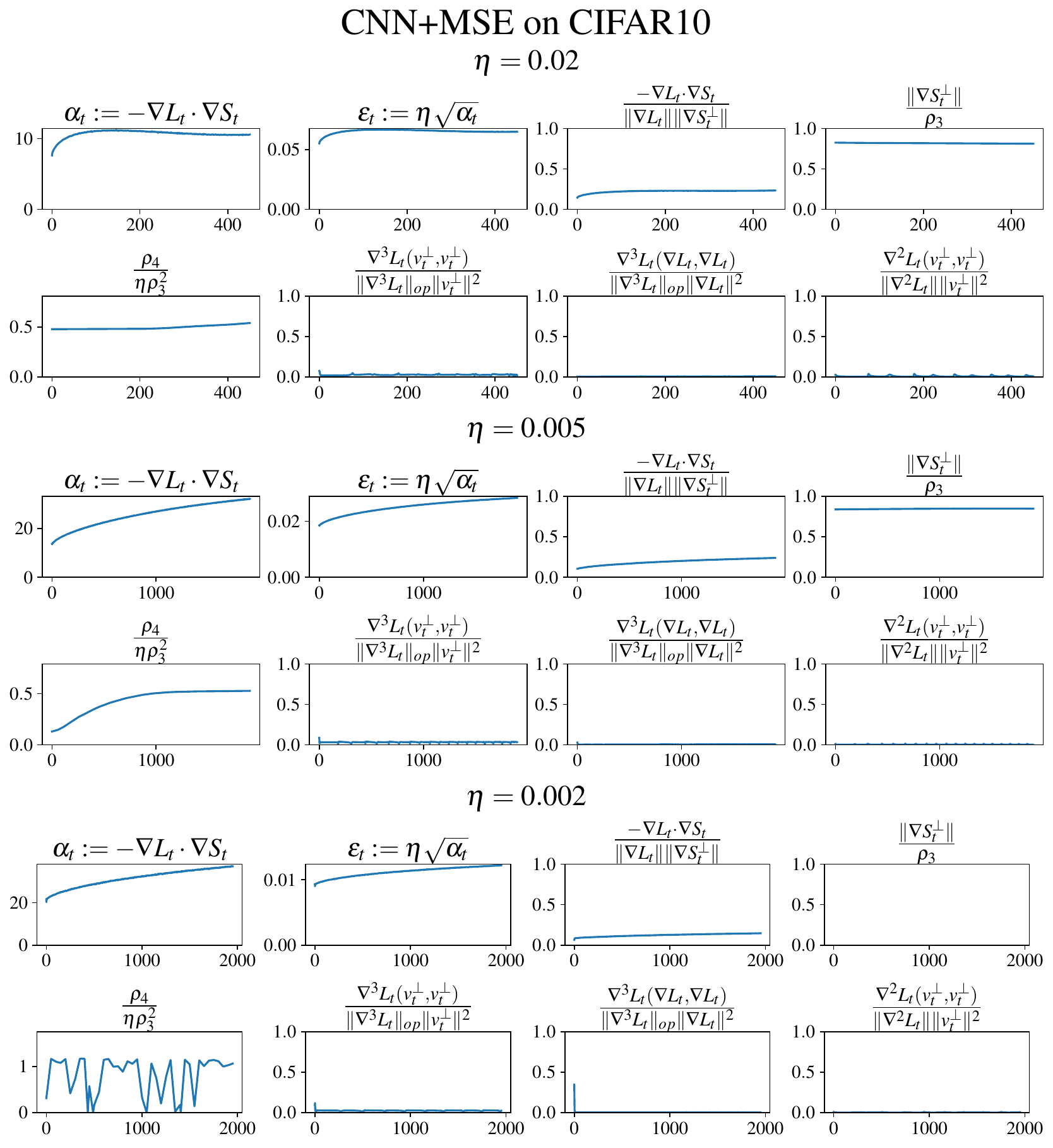}
\newpage
\includegraphics[width=\linewidth]{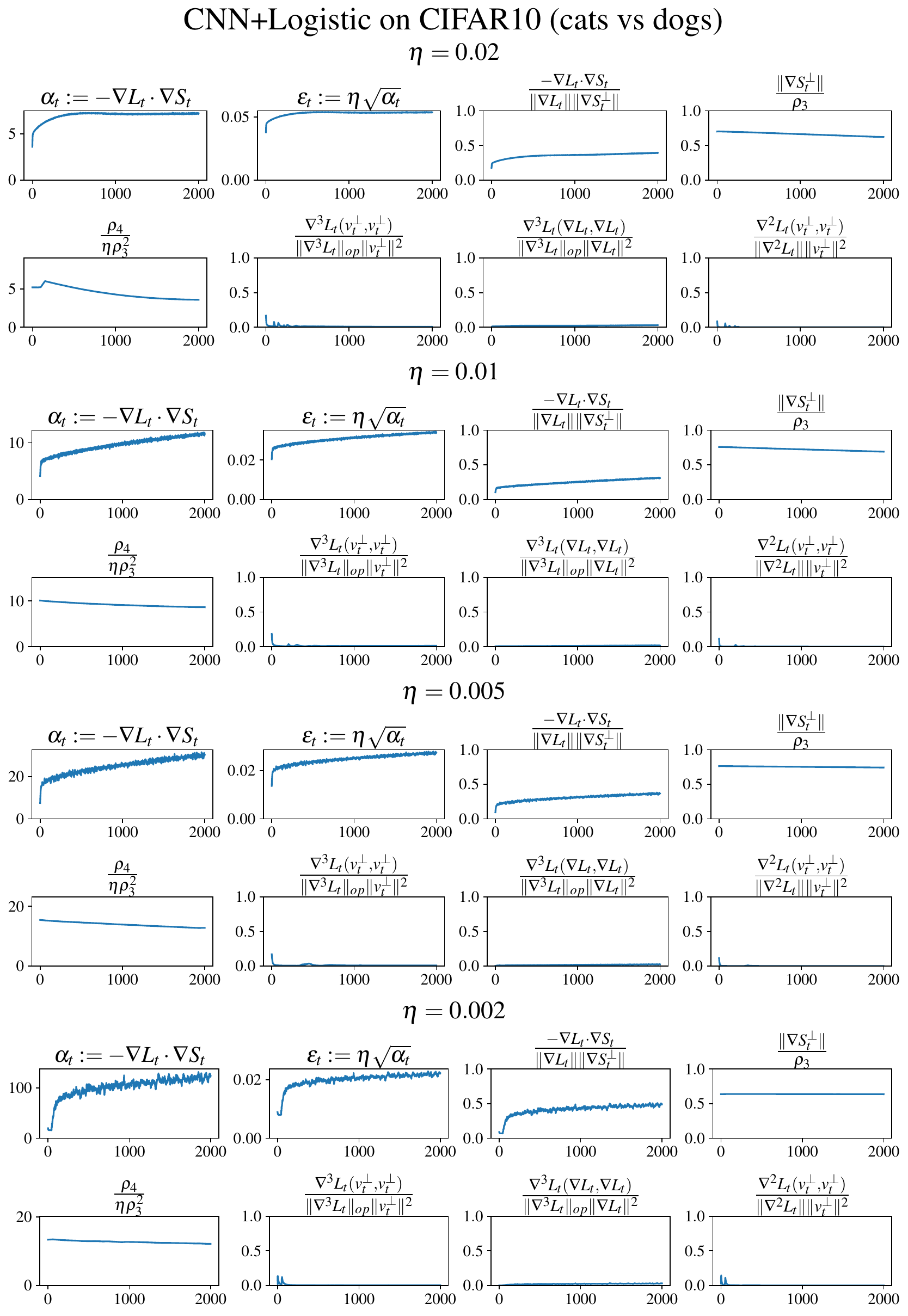}
\newpage
\includegraphics[width=\linewidth]{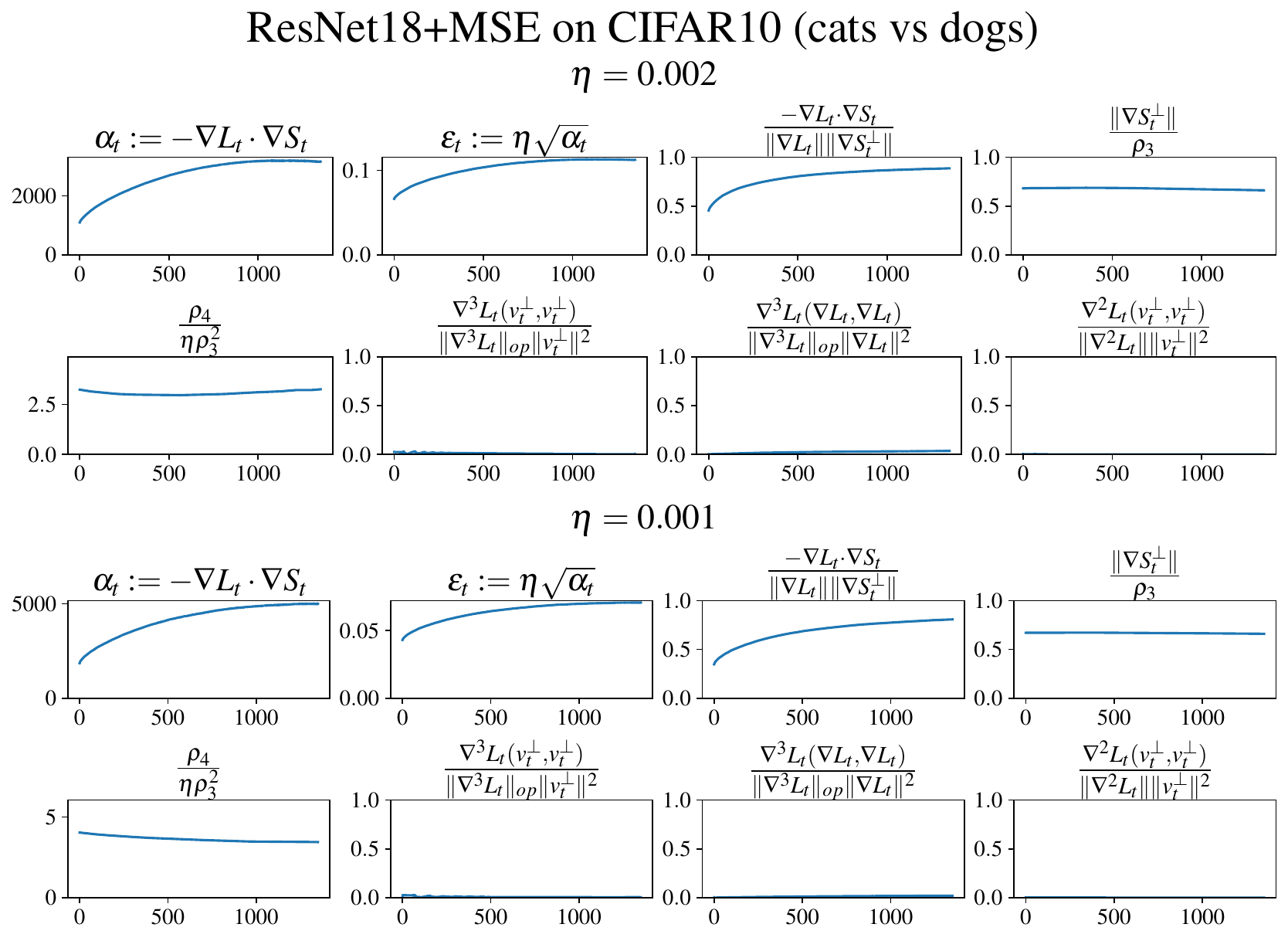}
\includegraphics[width=\linewidth]{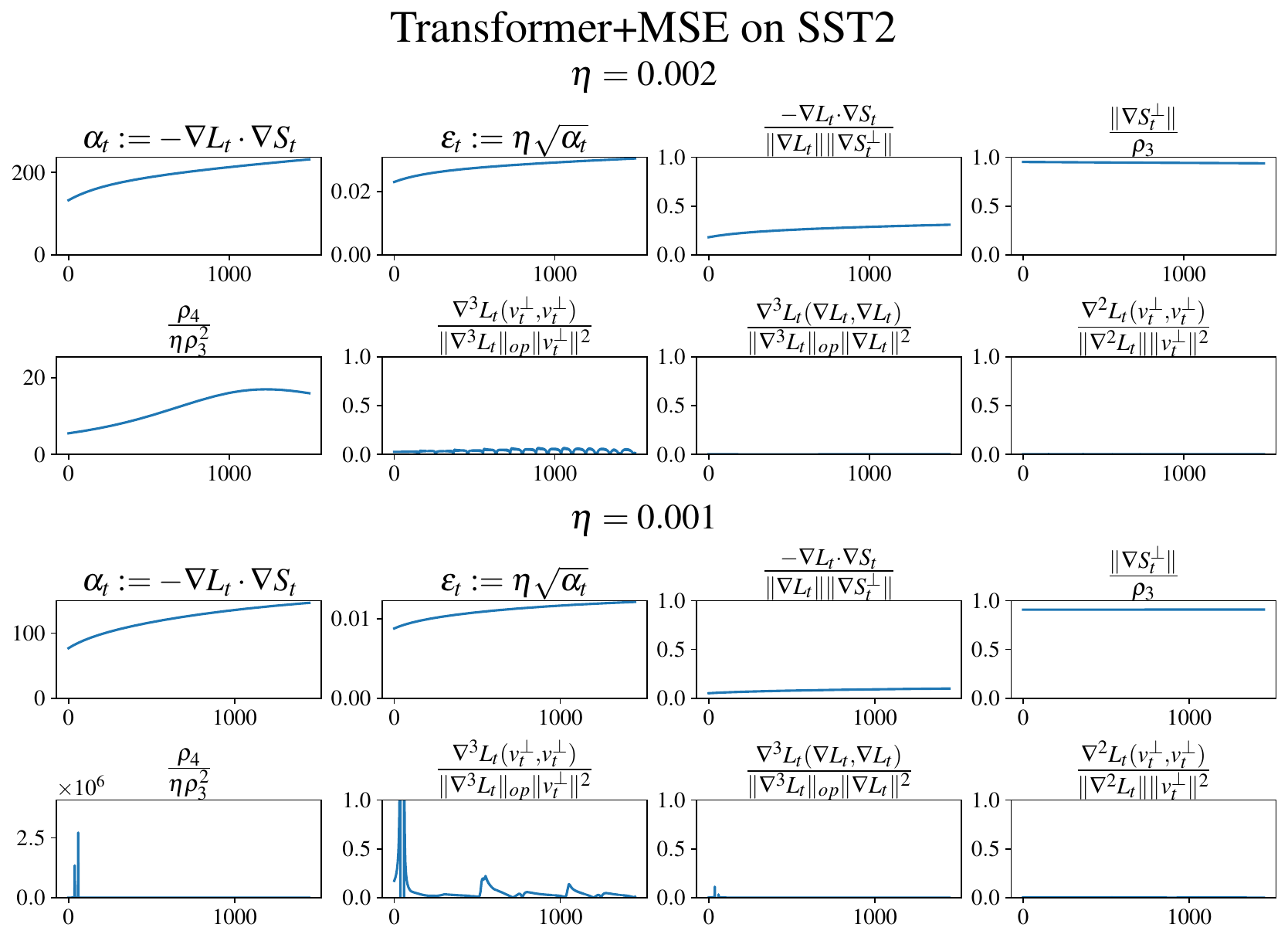}
\newpage
\includegraphics[width=\linewidth]{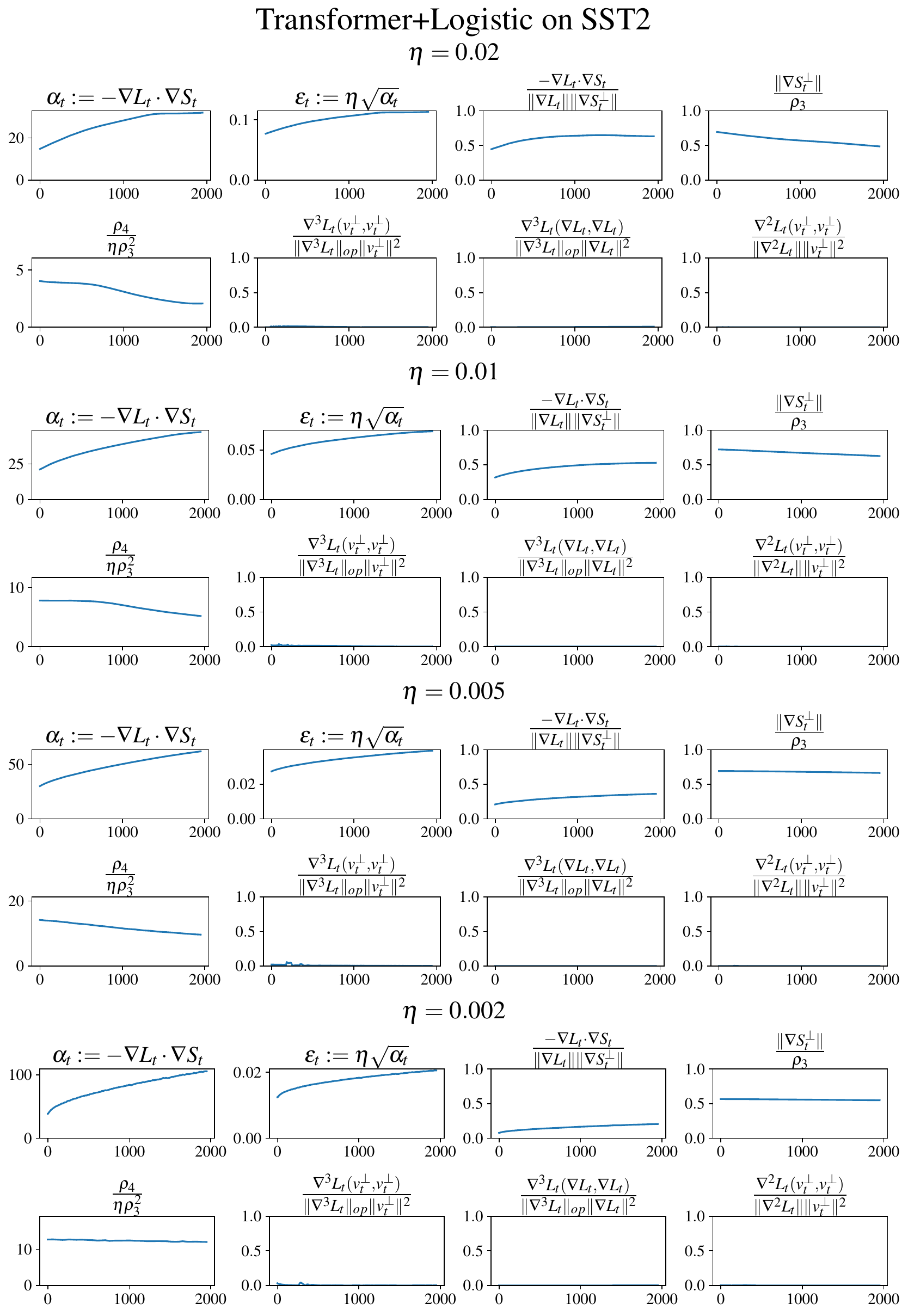}
\end{center}
\newpage

\section{The Generalized Predicted Dynamics}\label{sec:generalized_dynamics}

Our analysis relies on a cubic Taylor expansion of the gradient. However, in order for this Taylor expansion to accurately track the gradients we need a bound on the fourth derivative of the loss (\Cref{assume:rho4}). \Cref{sec:experiments} and \Cref{sec:verify_assumptions} show that this approximation is sufficient to capture the dynamics of gradient descent at the edge of stability for many standard models when the loss criterion is the mean squared error. However, for certain architectures and loss functions, including ResNet18 and models trained with the logistic loss, this condition is often violated.

In these situations, the loss function in the top eigenvector direction is either \emph{sub-quadratic}, meaning that the quadratic Taylor expansion overestimates the loss and sharpness\footnote{This sub-quadratic phenomenon was also observed in \citep{Ma2022TheMS}.}, or \emph{super-quadratic}, meaning that the quadratic Taylor expansion underestimates the loss and sharpness. To capture this phenomenon, we derive a more general form of the predicted dynamics which reduces to the standard predicted dynamics in \Cref{sec:theory} when the loss in the top eigenvector direction is approximately quadratic. In addition, \Cref{sec:additional_experiments} shows that the generalized predicted dynamics capture the dynamics of gradient descent at the edge of stability for both mean squared error and cross-entropy in all settings we tested.

\subsection{Deriving the Generalized Predicted Dynamics}

To derive the generalized predicted dynamics, we will abstract away the dynamics in the top eigenvector direction. Specifically, for every $t$ we define
\begin{align*}
    F_t(x) := L(\thetad_t + xu_t) - L(\thetad_t) - \frac{x^2}{\eta}.
\end{align*}
We say that $L$ is sub-quadratic at $t$ if $F_t(x) < 0$ and super-quadratic if $F_t(x) > 0$.

Note that knowing $F_t$ is not sufficient to capture the dynamics in the $u_t$ direction. Specifically,
\begin{align*}
    x_{t+1} = x_t - \eta u_t \cdot \nabla L(\thetad_t + v_t) \ne x_t - \eta u_t \cdot \nabla L(\thetad_t + x u_t).
\end{align*}
It is still critically important to track the effect that the movement in the $\nabla S_t^\perp$ direction has on the dynamics of $x$. As in \Cref{sec:four_stages}, the effect of the movement in the $\nabla S_t^\perp$ direction on the dynamics of $x$ is changing the sharpness by $y_t$. This gives us the generalized predicted dynamics update:
\begin{align*}
    v^*_{t+1} = P_{u_{t+1}}^\perp(I - \eta \nabla^2 L_t) P_{u_t}^\perp v^*_t + \eta P_{u_{t+1}}^\perp\nabla S_t^\perp \qty[\frac{\delta_t^2 - {x^*_t}^2}{2}] -x^\star_{t+1} \cdot u_{t+1}
\end{align*}
where $x^\star_{t+1} = -(1 + \eta y_t^\star) x^\star_t - \eta F'(x^\star_t).$

Note that when $F_t(x) = 0$ is exactly quadratic, this reduces to the standard predicted dynamics update in \cref{eq:predicted full}. Note that the update for $y$ is completely unchanged:
\begin{lemma}
    Restricted to the $u_t,\nabla S_t$ directions, the generalized predicted dynamics $v^\star_t$ imply:
    \begin{align}\label{eq:gen_pred_xy}
        x^\star_{t+1} = -(1 + \eta y_t^\star) x^\star_t - \eta F'(x^\star_t) \qand y^\star_{t+1} = \eta\sum_{s = 0}^t \beta_{s \to t}\qty[\frac{\delta_s^2 - {x^*_s}^2}{2}].
    \end{align}
\end{lemma}
The proof is identical to the proof of \Cref{lemma:predicted_x_y}.

\subsection{Properties of the Generalized Predicted Dynamics}
Note that due to the sign flipping argument in \Cref{sec:proofs}, we can assume that $F$ is an even function as the odd part will only influence the dynamics through additional oscillations of period 2, so throughout the remainder of this section we will assume that $F_t(x) = F_t(-x)$. Otherwise, we can simply redefine $F$ by its even part.

Next, note that the fixed point of \cref{eq:gen_pred_xy} is still when $x_t = \delta_t$, regardless of the shape of $F_t$, due to the need to stabilize the $\nabla S_t^\perp$ direction. This contradicts previous 1-dimensional analyses of edge of stability in which the fixed point in the top eigenvector direction strongly depends on the shape of $F_t$, the loss in the $u_t$ direction.

The limiting value of $y_t$ can therefore be read from the update for $x_t$. If $(\delta_t,y)$ is an orbit of period 2 of \cref{eq:gen_pred_xy}, then
\begin{align*}
    -\delta_t = -(1+\eta y)\delta_t - \eta F'(\delta_t) \implies y = -\frac{F'(\delta_t)}{\delta_t}.
\end{align*}
In addition, note that the sharpness can no longer be approximated as $S(\theta_t) \approx 2/\eta + y_t$ as the sharpness now changes along the $u_t$ direction. In particular, it changes by $F''(x)$ so that
\begin{align*}
    S(\theta_t) \approx 2/\eta + y_t + F''(x_t).
\end{align*}
Therefore, the limiting sharpness of \cref{eq:gen_pred_xy} is
\begin{align*}
    S(\theta_t) \to 2/\eta - \frac{F_t'(\delta_t)}{\delta_t} + F_t''(\delta_t).
\end{align*}
When $F_t = 0$ and the loss is exactly quadratic in the $u$ direction, this update reduces to fixed point predictions in \Cref{sec:four_stages}.

One interesting phenomenon observed by \citet{cohen2021eos} is that with cross-entropy loss, the sharpness was never exactly $2/\eta$, but usually hovered above it. This contradicts the predictions of the standard predicted dynamics which predict that the fixed point has sharpness $0$. However, using the generalized predicted dynamics \cref{eq:gen_pred_xy}, we can give a clear explanation.

When the loss is sub-quadratic, e.g. when $F_t(x) = -\rho_4 \frac{x^4}{24}$, we have
\begin{align*}
    S(\theta_t) \to 2/\eta + \rho_4 \frac{\delta_t^2}{6} - \rho_4 \frac{\delta_t^2}{2} = 2/\eta - \rho_4 \frac{\delta_t^2}{3} < 2/\eta
\end{align*}
so the sharpness will converge to a value \emph{below} $2/\eta$. On the other hand if the loss is super-quadratic, the sharpness converges to a value \emph{above} $2/\eta$. More generally, whether the loss converges to a value above or below $2/\eta$ depends on the sign of $F_t''(\delta_t) - \delta_t F_t'(\delta_t)$.

In our experiments in \Cref{sec:additional_experiments}, we observed both sub-quadratic and super-quadratic loss functions. In particular, the loss was usually sub-quadratic when it first reached instability but gradually became super-quadratic as training progressed at the edge of stability.

\section{Additional Experiments}\label{sec:additional_experiments}

\subsection{The Benefit of Large Learning Rates: Training Time and Generalization}

We trained ResNet18 with full batch gradient descent on the full 50k training set of CIFAR10 with various learning rates, in addition to the commonly proposed learning rate schedule $\eta_t := 1/S(\theta_t)$. We show that despite entering the edge of stability, large learning rates converge much faster. In addition, due to the self-stabilization effect of gradient descent, the final sharpness is bounded by $2/\eta$ which is smaller for larger learning rates and leads to better generalization (see \Cref{fig:training_speed}).

\begin{figure}[h]
    \centering
    \includegraphics[width=\textwidth]{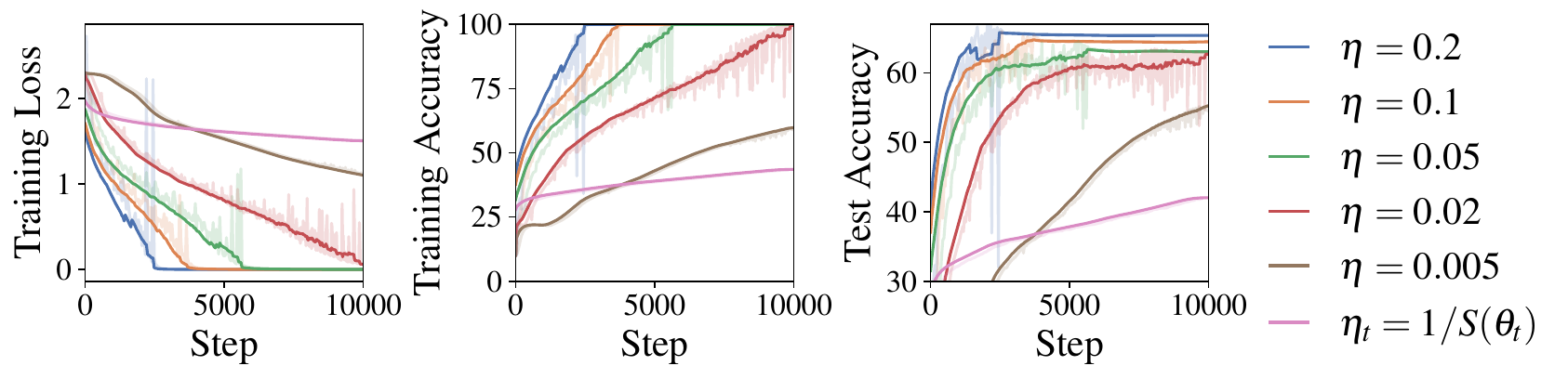}
    \caption{Large learning rates converge faster and generalize better (ResNet18 and CIFAR10).}
    \label{fig:training_speed}
\end{figure}

\begin{center}
\newpage
\subsection{Experiments with the Generalized Predicted Dynamics}
\includegraphics[width=\linewidth]{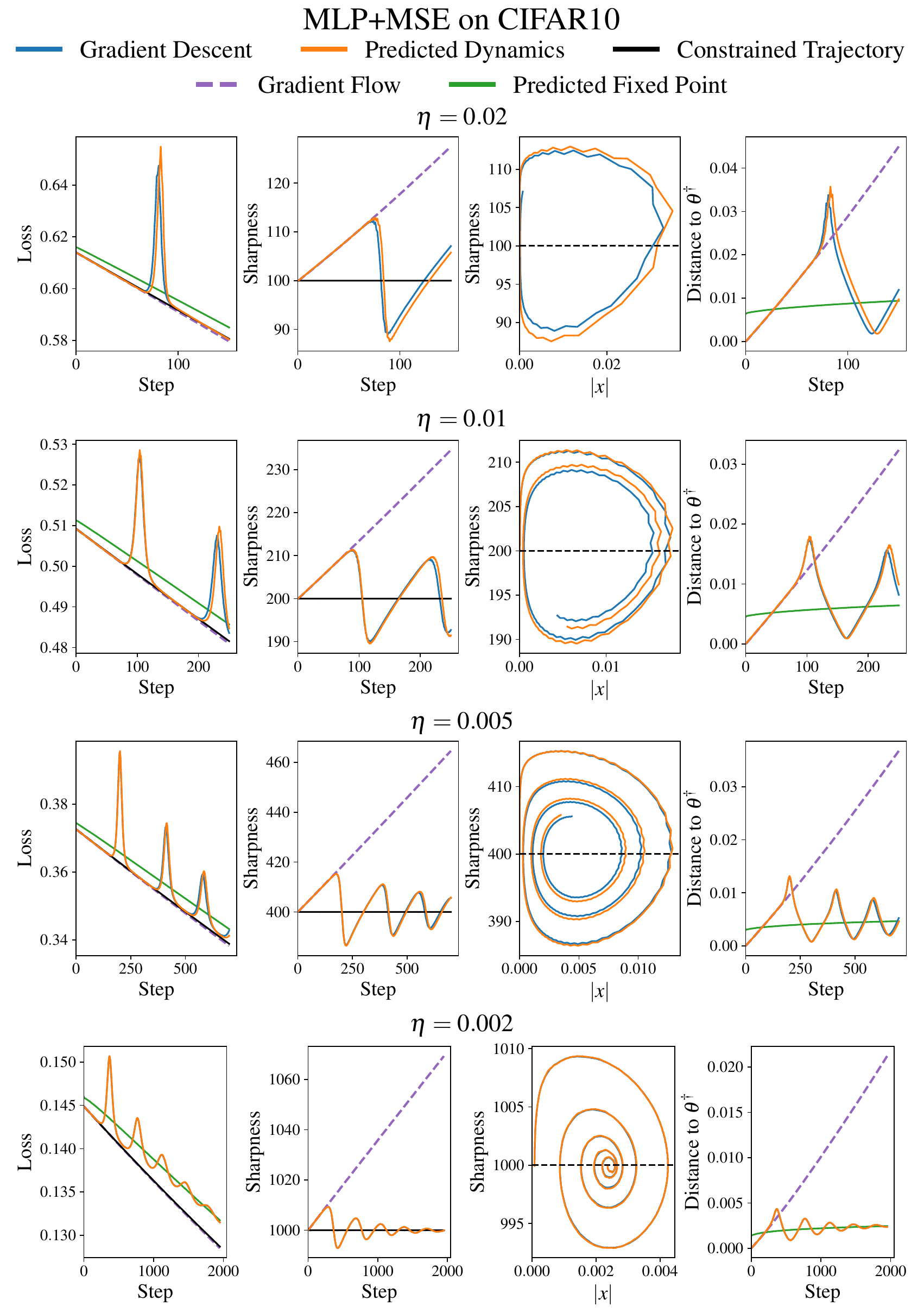}
\newpage
\includegraphics[width=\linewidth]{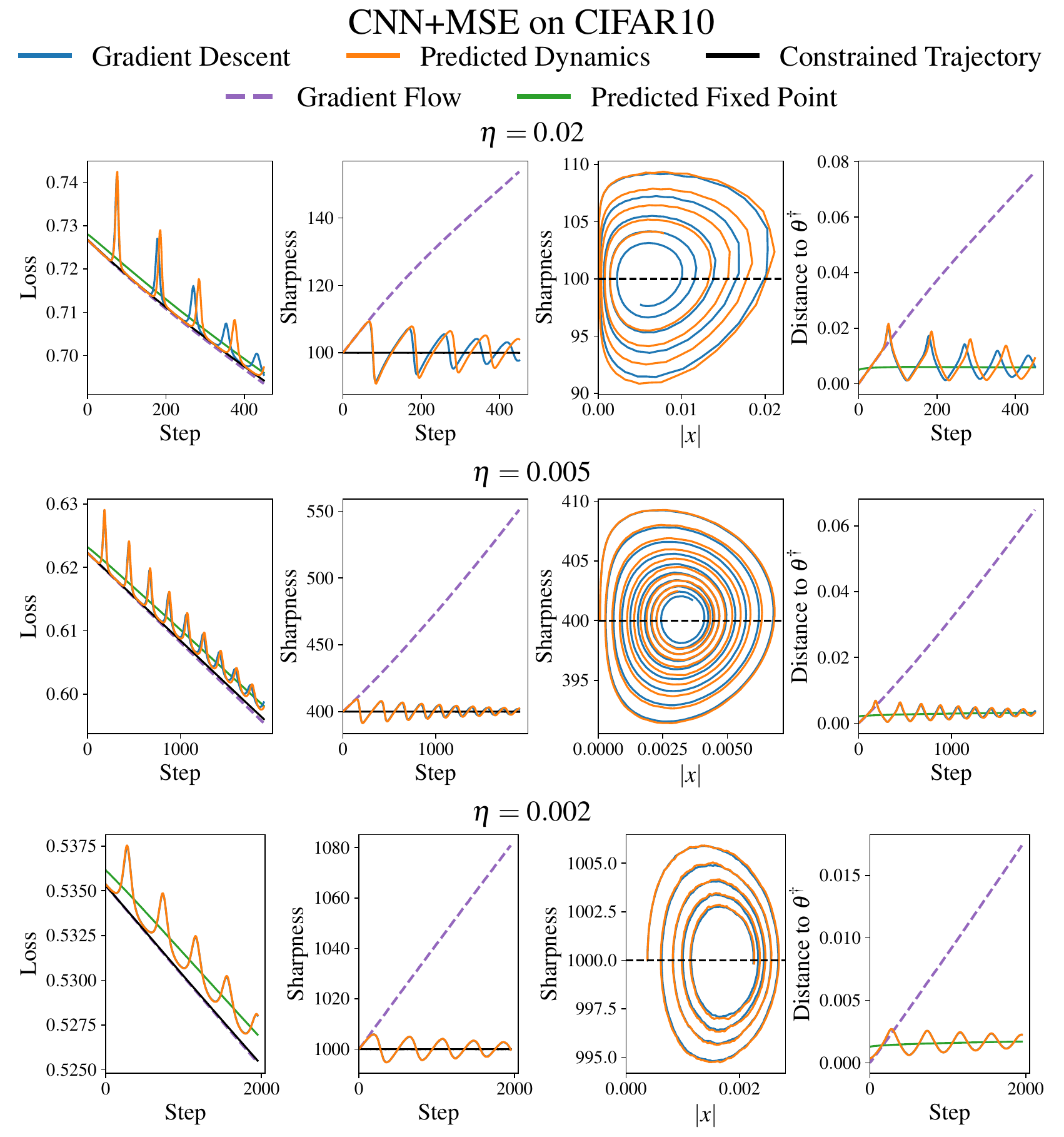}
\newpage
\includegraphics[width=\linewidth]{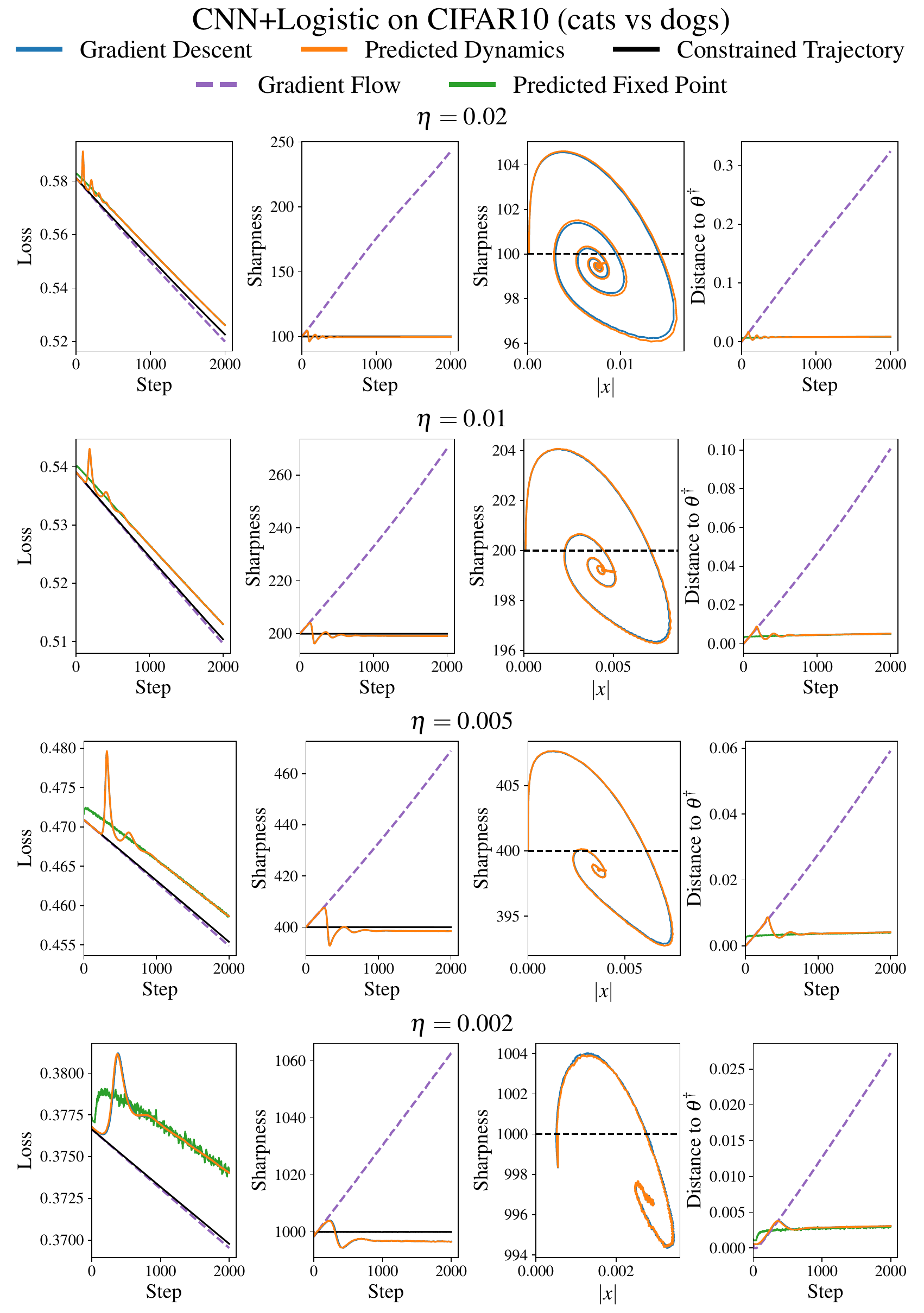}
\newpage
\includegraphics[width=\linewidth]{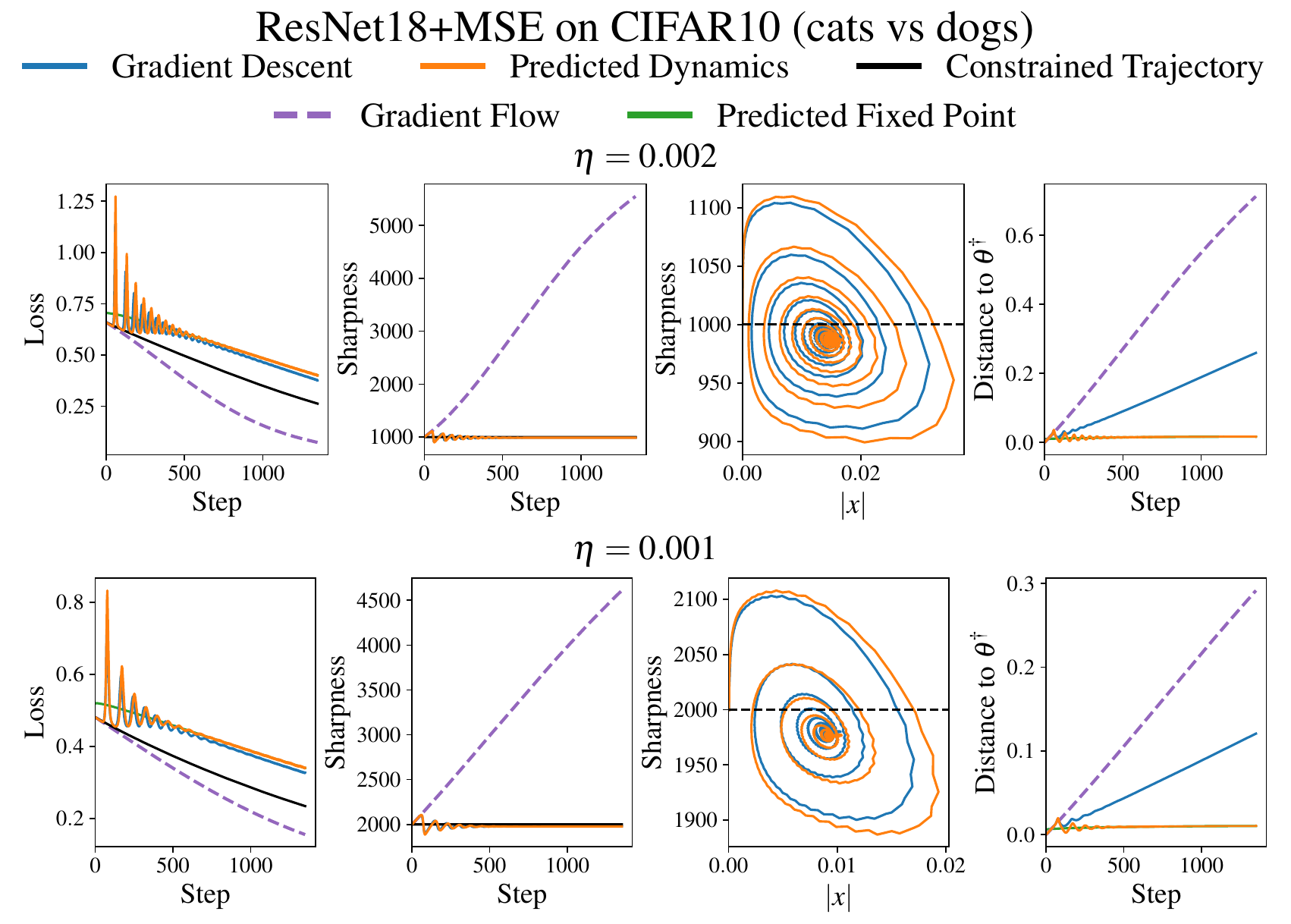}
\includegraphics[width=\linewidth]{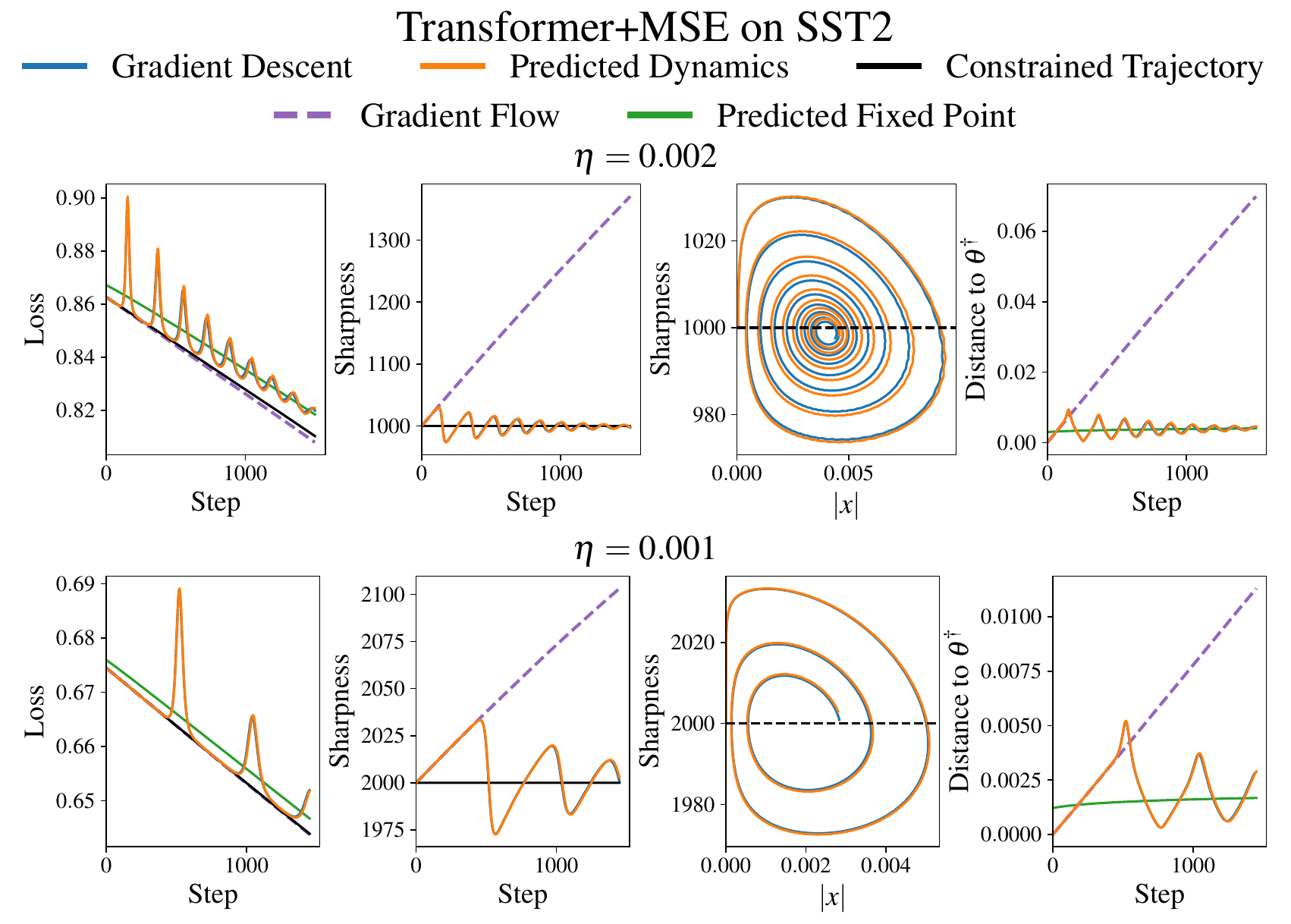}
\newpage
\includegraphics[width=\linewidth]{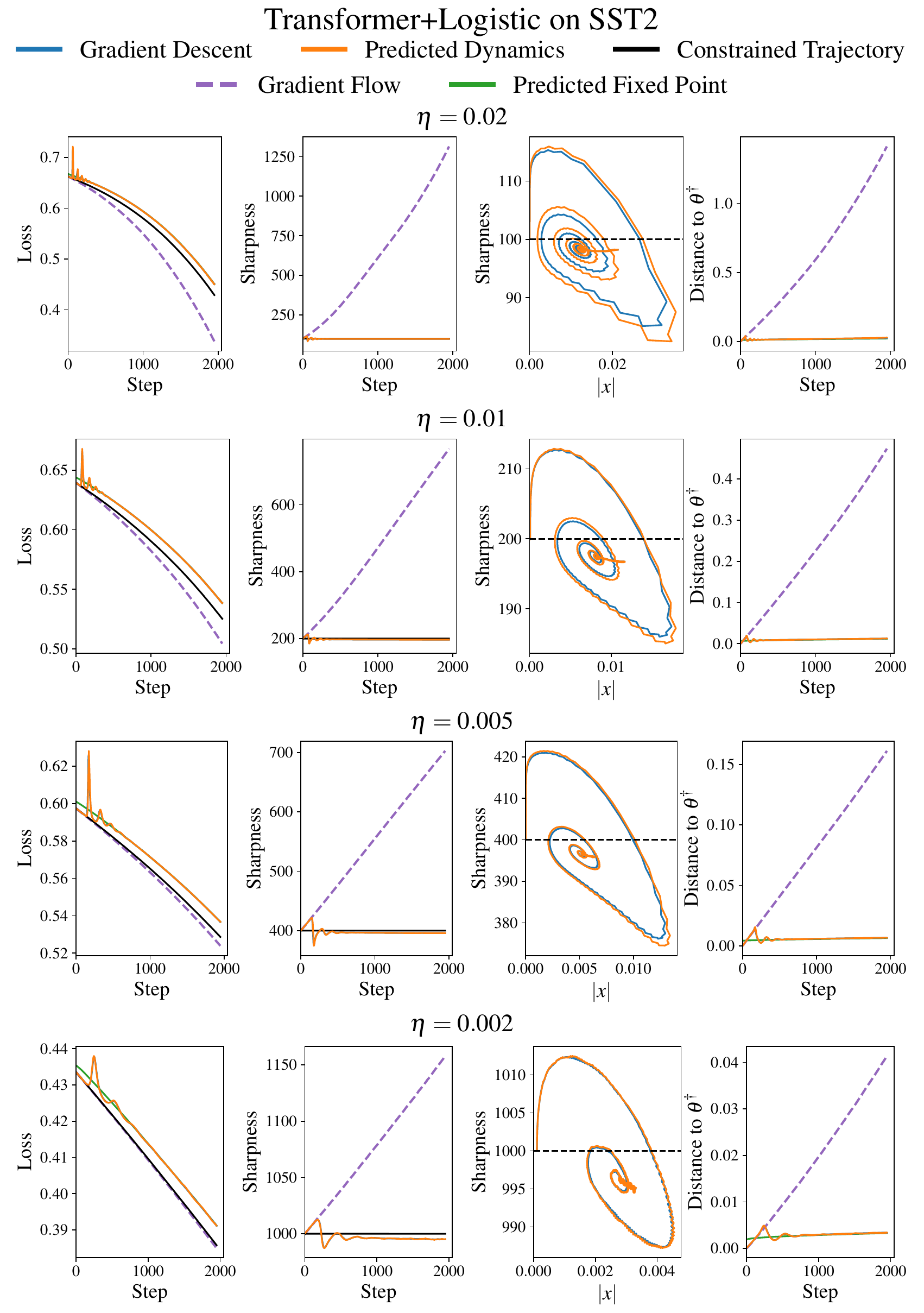}
\end{center}
\newpage

\section{Scale Invariant Losses}\label{sec:scale_invariant}

\begin{lemma}\label{lem:scale_invariant_assumptions}
    Let $f$ be a scale invariant loss function, i.e. $f(\theta) = f(c \theta)$. Let $L(\theta) = f(\theta) + \frac{\lambda}{2} \norm{\theta}^2$. Then for any local minimizer $\theta$ of $f(\theta)$ such that $S(\theta) = 2/\eta$,
    \begin{itemize}
        \item $\nabla L(\theta) \perp u(\theta)$
        \item $\rho_4 = O(\eta \rho_3^2)$
        \item $\alpha(\theta) > 0$
        \item $\frac{\alpha(\theta)}{\norm{\nabla L(\theta)}\norm{\nabla S(\theta)}} = \Theta(1)$
        \item $\norm{\nabla S(\theta)} = \Theta(\rho_3)$
    \end{itemize}
\end{lemma}

Our primary result is that gradient descent solves the constrained problem $\min_\theta L(\theta)$ such that $S(\theta) \le 2/\eta$. Let $S_f(\theta)$ denote the largest eigenvalue of $\nabla^2 f(\theta)$. To prove equivalence to the sharpness reduction, we will need the following lemma from \citep{Lyu2022Normalization} which follows from the scale invariance of the $f$:
\begin{align*}
    S_f(\theta) = \frac{1}{\|\theta\|^2} S_f(\theta/\|\theta\|).
\end{align*}
Let $\overline \theta := \frac{\theta}{\|\theta\|}$. Then we have the following equality between minimization problems:
\begin{align*}
    &\min_\theta L(\theta) \qq{such that} S(\theta) \le 2/\eta \\
    &\iff \min_\theta f(\overline \theta) + \lambda \frac{\norm{\theta}^2}{2} \qq{such that} S_f(\theta) \le 2/\eta - \lambda  \\
    &\iff \min_{\overline \theta,\|\theta\|} f(\overline \theta) + \lambda \frac{\norm{\theta}^2}{2} \qq{such that} \frac{1}{\|\theta\|^2}S_f(\overline \theta) \le \frac{2-\eta\lambda}{\eta} \\
    &\iff \min_{\overline \theta} f(\overline\theta) + \frac{\eta \lambda}{2-\eta \lambda} S_f(\overline \theta)
\end{align*}
where the last line follows from the scale-invariance of the loss function. In particular if $\eta \lambda$ is sufficiently small and the dynamics are initialized near a global minimizer of the loss, this will converge to the solution of the constrained problem:
\begin{align*}
    \min_{\|\overline \theta\| = 1} S_f(\overline \theta) \qq{such that} f(\overline \theta) = 0.
\end{align*}

\subsection{Scale Invariant Lemmas}

Let $\theta$ denote an arbitrary parameter and let $\overline \theta = \theta/\|\theta\|$. Throughout this section, let $f$ be a scale invariant function with non-vanishing Hessian.

\begin{lemma}
    \begin{align*}
        \nabla f(\theta) &= \frac{P_\theta^\perp \nabla f(\overline \theta)}{\norm{\theta}} \\
        \nabla^2 f(\theta) &= \frac{P_\theta^\perp \nabla^2 f(\overline \theta) P_\theta^\perp - P_\theta^\perp \nabla f(\overline \theta) \overline \theta^T - \overline \theta (P_\theta^\perp \nabla f(\overline \theta))^T}{\norm{\theta}^2}
    \end{align*}
\end{lemma}
\begin{proof}
    We start with the equality:
    \begin{align*}
        f(\theta) = f(\overline \theta).
    \end{align*}
    Differentiating with respect to $\theta$ and using that $\nabla_\theta \overline\theta = \frac{P_{\overline \theta}^\perp}{\norm{\theta}}$ gives,
    \begin{align*}
        \nabla f(\theta) = \frac{P_\theta^\perp \nabla f(\overline \theta)}{\norm{\theta}}.
    \end{align*}
    Differentiating this again gives:
    \begin{align*}
        \nabla^2 f(\theta) = \frac{P_\theta^\perp \nabla^2 f(\overline \theta) P_\theta^\perp - P_\theta^\perp \nabla f(\overline \theta) \overline \theta^T - \overline \theta (P_\theta^\perp \nabla f(\overline \theta))^T}{\norm{\theta}^2}.
    \end{align*}
\end{proof}

A few corollaries immediately follow:
\begin{corollary}\label{cor:si:hessian_at_critical_point}
    For any critical point $\theta$ of $f$,
    \begin{align*}
        \nabla^2 f(\theta) = \frac{P_\theta^\perp \nabla^2 f(\overline \theta) P_\theta^\perp}{\norm{\theta}^2}.
    \end{align*}
\end{corollary}
\begin{corollary}\label{cor:si:u_perp_dL}
    For any critical point $\theta$ of $f$,
    \begin{align*}
        u(\theta) \perp \nabla L(\theta)
    \end{align*}
\end{corollary}
\begin{proof}
    Note that from \Cref{cor:si:hessian_at_critical_point}, the top eigenvector of $\nabla^2 f$ is perpendicular to $\theta$. In addition,
    \begin{align*}
        \nabla^2 L(\theta) = \nabla^2 f(\theta) + \lambda I
    \end{align*}
    so this is also the top eigenvector of $\nabla^2 L(\theta)$, i.e. $u(\theta)$. Finally,
    \begin{align*}
        \nabla L(\theta) = \nabla f(\theta) + \lambda \theta = \lambda \theta
    \end{align*}
    which is parallel to $\theta$ and concludes the proof.
\end{proof}

\begin{lemma}
    \begin{align*}
        \nabla S(\theta) = \frac{P_{\overline \theta}^\perp \nabla S(\overline \theta) - (S(\overline\theta) - \lambda)\overline\theta}{\norm{\theta}^3}
    \end{align*}
\end{lemma}
\begin{proof}
    Let $S_f(\theta)$ denote the largest eigenvalue of $\nabla^2 f(\theta)$. Then by scale invariance, $\nabla^2 f(\theta) = \nabla^2 f(\overline \theta)/\norm{\theta}^2$. This implies that
    \begin{align*}
        S_f(\theta) = \frac{S_f(\overline \theta)}{\norm{\theta}^2}.
    \end{align*}
    Differentiating this gives:
    \begin{align*}
        \nabla S_f(\theta) = \frac{P_{\overline \theta}^\perp \nabla S_f(\overline \theta) - S_f(\overline \theta) \overline \theta}{\norm{\theta}^3}.
    \end{align*}
    Finally, we have from $\nabla^2 L(\theta) = \nabla^2 f(\theta) + \lambda I$ that $S(\theta) = S_f(\theta) + \lambda$ so
    \begin{align*}
        \nabla S(\theta) = \frac{P_{\overline \theta}^\perp \nabla S(\overline \theta) - (S(\overline\theta) - \lambda)\overline\theta}{\norm{\theta}^3}
    \end{align*}
\end{proof}

\begin{lemma}\label{lem:si:rho3_rho4}
    Let
    \begin{align*}
       \rho_3 = \sup_{S(\theta) \le 4/\eta} \norm{\nabla^3 L(\theta)}_{op} \qand \rho_4 = \sup_{S(\theta) \le 4/\eta} \norm{\nabla^4 L(\theta)}_{op}.
    \end{align*}
    Then
    \begin{align*}
        \rho_3 = \Theta(\eta^{-3/2}) \qand \rho_4 = O(\eta^{-2}).
    \end{align*}
    In particular, $\rho_4 = O(\eta \rho_3^2)$.
\end{lemma}
\begin{proof}
    Note that $S(\theta) < 4/\eta$ implies that
    \begin{align*}
        4/\eta = \frac{S_f(\overline\theta)}{\norm{\theta}^2} + \lambda \implies \norm{\theta} = \Theta{\eta^{-1/2}}.
    \end{align*}
    Therefore,
    \begin{align*}
        \nabla^3 L(\theta) = \frac{\nabla^3 L(\overline \theta)}{\norm{\theta}^3} = O(\theta^{-3/2})
    \end{align*}
    and
    \begin{align*}
        \nabla^4 L(\theta) = \frac{\nabla^4 L(\overline \theta)}{\norm{\theta}^4} = O(\theta^{-2})
    \end{align*}
    by compactness of $S^{d-1}$. In addition, note that $\norm{\nabla^3 L(u,u)} = \norm{\nabla S(\theta)} \ge \frac{S(\overline \theta) - \lambda}{\norm{\theta}^3} = \Theta(\eta^{-3/2})$ so $\rho_3 = \Theta(\eta^{-3/2})$.
\end{proof}

\begin{lemma}\label{lem:si:alpha_positive}
    At any second order stationary point $\theta$ of $f$,
    \begin{align*}
        \alpha(\theta) = \frac{\lambda(S(\overline \theta) - \lambda)}{\norm{\theta}^2} > 0.
    \end{align*}
\end{lemma}
\begin{proof}
    \begin{align*}
        \alpha(\theta) = -\nabla L(\theta) \cdot \nabla S(\theta) = -\lambda \theta \cdot \qty[-\frac{(S(\overline \theta) - \lambda)\overline \theta}{\norm{\theta}^3}] = \frac{\lambda(S(\overline \theta) - \lambda)}{\norm{\theta}^2}.
    \end{align*}
\end{proof}

\begin{lemma}\label{lem:si:alpha_ratio}
    At any second order stationary point $\theta$ of $f$,
    \begin{align*}
        \frac{\alpha(\theta)}{\norm{\nabla L(\theta)}\norm{\nabla S(\theta)}} = \frac{\lambda(S(\overline \theta) - \lambda)}{\norm{\theta}^2} = \Theta(1).
    \end{align*}
\end{lemma}
\begin{proof}
    \begin{align*}
        \frac{\alpha(\theta)}{\norm{\nabla L(\theta)}\norm{\nabla S(\theta)}}
        &= \frac{\lambda(S(\overline \theta) - \lambda)}{\norm{\theta}^2} \cdot \frac{1}{\lambda \norm{\theta}} \cdot \frac{\norm{\theta}^3}{\sqrt{\norm{P_{\overline\theta}^\perp \nabla S(\overline\theta)}^2 + (S(\overline\theta)-\lambda)^2}} \\
        &= \frac{1}{\sqrt{1 + \norm{\frac{P_{\overline \theta}^\perp \nabla S(\overline\theta)}{S(\overline\theta)-\lambda}}^2}} \\
        &= \Theta(1)
    \end{align*}
    where the last step follows from compactness of $S^{d-1}$ and the fact that $\nabla^2 f$ is non-vanishing.
\end{proof}

\begin{proof}[Proof of \Cref{lem:scale_invariant_assumptions}]
    The lemma is simply a restatement of \Cref{cor:si:u_perp_dL}, \Cref{lem:si:alpha_positive}, \Cref{lem:si:rho3_rho4}, and \Cref{lem:si:alpha_ratio}.
\end{proof}

\section{Proofs}\label{sec:proofs}

\subsection{Properties of the Constrained Trajectory}
We next prove several nice properties of the constrained trajectory. Before, we require the following auxiliary lemma, which shows that several quantities are Lipschitz in a neighborhood around the constrained trajectory:

\begin{definition}[Lipschitz Sets]
    $\mathcal{S}_t := B(\thetad_t,\frac{2-c}{4 \eta \rho_3})$ where $c$ is the constant in \Cref{assumption:eigval_gap} and $B(x,r)$ denotes the ball of radius $r$ centered at $x$.
\end{definition}

\begin{lemma}[Lipschitz Properties]\label{lem:lipschitz}\leavevmode
    \begin{enumerate}
        \item $\theta \to \nabla L(\theta)$ is $O(\eta^{-1})$-Lipschitz in each set $\mathcal{S}_t$.
        
        \item $\theta \to \nabla^2 L(\theta)$ is $\rho_3$-Lipschitz with respect to $\norm{\cdot}_2$.
        
        \item $\theta \to \lambda_i(\nabla^2 L(\theta))$ is $\rho_3$-Lipschitz.
        
        \item $\theta \to u(\theta)$ is $O(\eta \rho_3)$-Lipschitz in each set $\mathcal{S}_t$.
        
        \item $\theta \to \nabla S(\theta)$ is $O(\eta \rho_3^2)$-Lipschitz in each set $\mathcal{S}_t$.
    \end{enumerate}
\end{lemma}
\begin{proof}
    The Lipschitzness of $\nabla^2 L(\theta)$ follows immediately from the bound $\norm{\nabla^3 L(\theta)}_{op} \le \rho_3$. Weil's inequality then immediately implies the desired bound on the Lipschitz constant of the eigenvalues of $\nabla^2 L(\theta)$. Therefore for any $t$, we have for all $\theta \in \mathcal{S}_t$:
    \begin{align*}
        \lambda_1(\nabla^2 L(\theta)) - \lambda_2(\nabla^2 L(\theta)) \ge \lambda_1(\nabla^2 L(\theta)) - \lambda_2(\nabla^2 L(\theta)) - 2 \rho_3 \frac{2-c}{4\eta\rho_3} \ge \frac{2-c}{2\eta}.
    \end{align*}
    Next, from the derivative of eigenvector formula:
    \begin{align*}
        \norm{\nabla u(\theta)}_2
        &= \norm{(\lambda_1(\nabla^2 L(\theta)) I - \nabla^2 L(\theta))^\dagger \nabla^3 L(\theta)(u(\theta))}_2 \\
        &\le \frac{\rho_3}{\lambda_1(\nabla^2 L(\theta)) - \lambda_2(\nabla^2 L(\theta))} \\
        &\le \frac{2\eta\rho_3}{2-c} \\
        &= O(\eta \rho_3)
    \end{align*}
    which implies the bound on the Lipschitz constant of $u$ restricted to $\mathcal{S}_t$. Finally, because $\nabla S(\theta) = \nabla^3 L(\theta)(u(\theta),u(\theta))$,
    \begin{align*}
        \norm{\nabla^2 S(\theta)}_2 \le \norm{\nabla^4 L(\theta)}_{op} + 2\norm{\nabla^3 L(\theta)}_{op}\norm{\nabla u(\theta)}_2 \le O(\rho_4 + \eta \rho_3^2) \le O(\eta \rho_3^2)
    \end{align*}
    where the second to last inequality follows from the bound on $\norm{\nabla u(\theta)}_2$ restricted to $\mathcal{S}_t$ and the last inequality follows from \Cref{assume:rho4}.
\end{proof}

\begin{lemma}[First-order approximation of the constrained trajectory update $\{\thetad_t\}$]\label{lem:dagger_step}
    For all $t \le \mathscr{T}$,
    \begin{align*}
        \thetad_{t+1} = \thetad_t - \eta P_{u_t,\nabla S_t}^\perp \nabla L_t + O\qty(\epsilon^2\cdot \eta\norm{\nabla L_t}) \qand S_t = 2/\eta.
    \end{align*}
\end{lemma}
\begin{proof}
    We will prove by induction that $S_t = 2/\eta$ for all $t$. The base case follows from the definitions of $\theta_0,\thetad_0$. Next, assume $S(\thetad_t) = 0$ for some $t \ge 0$. Let $\theta' = \thetad_t - \eta \nabla L_t$. Then because $\thetad_t \in \mathscr{M}$ we have $\norm{\thetad_{t+1}-\theta'} \le \norm{\thetad_{t}-\theta'} = \eta \norm{\nabla L_t}$. Then because $\thetad_{t+1} = \proj_{\mathcal{M}}(\theta')$, the KKT conditions for this minimization problem imply that there exist $x,y$ with $y \ge 0$ such that
    \begin{align*}
        \thetad_{t+1}
        &= \thetad_t - \eta \nabla L_t - x \nabla_\theta \qty[\nabla L(\theta) \cdot u(\theta)]\eval_{\theta = \thetad_{t+1}} - y \nabla S_{t+1} \\
        &= \thetad_t - \eta \nabla L_t - x \qty[S_{t+1}u_{t+1} + \nabla u_{t+1}^T \nabla L_{t+1}] - y \nabla S_{t+1} \\
        &= \thetad_t - \eta \nabla L_t - x \qty[S_{t+1} u_{t+1} + O(\eta \rho_3 \norm{\nabla L_{t+1}})] - y\nabla S_{t+1} \\
        &= \thetad_t - \eta \nabla L_t - x \qty[S_t u_t + O(\eta \rho_3 \norm{\nabla L_{t}})] - y\qty[\nabla S_t + O(\eta^2 \rho_3^2 \norm{\nabla L_t})] \\
        &= \thetad_t - \eta \nabla L_t - x S_t u_t - y \nabla S_t + O\qty((\abs{x} \eta \rho_3 + \abs{y} \eta^2 \rho_3^2) \norm{\nabla L_{t}}).
    \end{align*}
    Next, note that we can decompose $\nabla S_t = u_t (\nabla S_t \cdot u_t) + \nabla S_t^\perp$:
    \begin{align*}
        \thetad_{t+1}
        &= \thetad_t - \eta \nabla L_t - \qty[x S_t + y (\nabla S_t \cdot u_t)] u_t - y \nabla S_t^\perp + O\qty((\abs{x} \eta \rho_3 + \abs{y}\eta^2 \rho_3^2) \norm{\nabla L_{t}}).
    \end{align*}
    Let $s_t = \frac{\nabla S_t^\perp}{\norm{\nabla S_t^\perp}}$. We can now perform the change of variables
    \begin{align*}
        (x',y') = \qty(x S_t + y(\nabla S_t \cdot u_t),y\norm{\nabla S_t^\perp}) \qc (x,y) = \qty(\frac{x'-y'\frac{\nabla S_t \cdot u_t}{\norm{\nabla S_t^\perp}}}{S_t},\frac{y'}{\norm{\nabla S_t^\perp}})
    \end{align*}
    to get 
    \begin{align*}
        \thetad_{t+1} = \thetad_t - \eta \nabla L_t - x' u_t - y' s_t + O\qty(\eta^2 \rho_3\norm{\nabla L} (\abs{x'} + \abs{y'})).
    \end{align*}
    Note that
    \begin{align}
        O(\eta^2 \rho_3 \norm{\nabla L}(\abs{x} + \abs{y})) \le \frac{\sqrt{x^2 + y^2}}{2}
    \end{align}
    for sufficiently small $\epsilon$ so because $\norm{\thetad_{t+1} - \theta'} \le \eta \norm{\nabla L_t}$ we have
    \begin{align*}
        \frac{\sqrt{x^2 + y^2}}{2} \le \norm{\thetad_{t+1} - \theta'} \le \eta \norm{\nabla L_t}
    \end{align*}
    so $x,y = O(\eta \norm{\nabla L_t})$. Therefore,
    \begin{align*}
        \thetad_{t+1} &= \thetad_t - \eta \nabla L_t - x' u_t - y' s_t + O\qty(\eta^3 \rho_3\norm{\nabla L}^2)\\
        &= \thetad_t - \eta \nabla L_t - x' u_t - y' s_t + O\qty(\epsilon^2\cdot \eta\norm{\nabla L_t})
    \end{align*}
    Then Taylor expanding $\nabla L_{t+1}$ around $\thetad_t$ gives
    \begin{align*}
        \nabla L_{t+1} \cdot u_{t+1}
        &= \nabla L_t \cdot u_t + (\nabla L_{t+1}-\nabla L_t) \cdot u_t + \nabla L_{t+1} \cdot (u_{t+1}-u_t) \\
        &= u_t^T \nabla^2 L_t \qty[-\eta \nabla L_t - x' u_t - y' s_t + O(\epsilon^2\cdot \eta\norm{\nabla L_t}] + O\qty(\epsilon^2\cdot \norm{\nabla L_t}) \\
        &= - x' S_t + O\qty(\epsilon^2\cdot \norm{\nabla L_t})
    \end{align*}
    so $x' = O(\epsilon^2\cdot \eta\norm{\nabla L_t})$. We can also Taylor expand $S_{t+1}$ around $\thetad_t$ and use that $S_t = 2/\eta$ to get
    \begin{align*}
        S_{t+1}
        &= 2/\eta + \nabla S_t \cdot \qty[-\eta \nabla L_t - x' u_t - y' s_t + O\qty(\eta^3 \rho_3 \norm{\nabla L_t}^2)] + O\qty(\epsilon^2\cdot \rho_3\eta\norm{\nabla L_t}) \\
        &= 2/\eta + \eta \alpha_t - y' \|\nabla S_t^\perp\| + O\qty(\epsilon^2\cdot \rho_3\eta\norm{\nabla L_t}).
    \end{align*}
    Now note that for $\epsilon$ sufficiently small we have
    \begin{align*}
        O\qty(\epsilon^2\cdot \rho_3\eta\norm{\nabla L_t}) \le O\qty(\epsilon^2\cdot \eta\alpha_t) \le \eta \alpha_t.
    \end{align*}
    Therefore if $y'=0$, we would have $S_{t+1} > 2/\eta$ which contradicts $\thetad_{t+1} \in \mathcal{M}$. Therefore $y' > 0$ and therefore $y > 0$, which by complementary slackness implies $S_{t+1} = 2/\eta$. This then implies that
    \begin{align*}
        -\eta \nabla L_t \cdot \nabla S_t^\perp - y' \|\nabla S_t^\perp\| + O(\epsilon^2\cdot \rho_3\eta\norm{\nabla L_t}) = 0 \implies y' = -\eta \nabla L_t \cdot \frac{\nabla S_t^\perp}{\norm{\nabla S_t^\perp}} + O\qty(\epsilon^2\cdot \eta\norm{\nabla L_t}).
    \end{align*}
    Putting it all together gives
    \begin{align*}
        \thetad_{t+1}
        &= \thetad_t - \eta P_{\nabla S_t^\perp}^\perp \nabla L_t + O\qty(\epsilon^2\cdot \eta\norm{\nabla L_t}) \\
        &= \thetad_t - \eta P_{u_t,\nabla S_t}^\perp \nabla L_t + O\qty(\epsilon^2\cdot \eta\norm{\nabla L_t})
    \end{align*}
    where the last line follows from $u_t \cdot \nabla L_t = 0$.
\end{proof}

\begin{lemma}[Descent Lemma for $\theta^\dagger$]\label{lem:dagger_descent}
    For all $t \le \mathscr{T}$,
    \begin{align*}
        L(\thetad_{t+1}) \le L(\thetad_t) - \Omega\qty(\eta\norm{P_{u_t,\nabla S_t}^\perp \nabla L_t}^2).
    \end{align*}
\end{lemma}
\begin{proof}
    Taylor expanding $L(\thetad_{t+1})$ around $L(\thetad_t)$ and using \Cref{lem:dagger_step} gives
    \begin{align*}
        L(\thetad_{t+1})
        &= L(\thetad_t) + \nabla L_t \cdot (\thetad_{t+1}-\thetad_t) + \frac{1}{2}(\thetad_{t+1}-\thetad_t)^T \nabla^2 L_t (\thetad_{t+1}-\thetad_t) + O\qty(\rho_3 \norm{\thetad_{t+1}-\thetad_t}^3) \\
        &= L(\thetad_t) - \eta \norm{P_{u_t,\nabla S_t}^\perp \nabla L_t}^2 + \frac{\eta^2 \lambda_2(\nabla^2 L_t) \norm{P_{u_t,\nabla S_t}^\perp \nabla L_t}^2}{2} + O\qty(\eta^3 \rho_3 \norm{\nabla L_t}^3) \\
        &= L(\theta_t^\dagger) - \frac{\eta (2-c)}{2} \norm{P_{u_t,\nabla S_t}^\perp \nabla L_t}^2 + O\qty(\eta^3 \rho_3 \norm{\nabla L_t}^3).
    \end{align*}
    Next, note that because $\gamma_t = \Theta(1)$ we have $\|\nabla L_t\| = O(\norm{P_{u_t,\nabla S_t}^\perp \nabla L_t})$.Therefore for $\epsilon$ sufficiently small,
    \begin{align*}
        O\qty(\eta^3 \rho_3 \norm{\nabla L_t}^3) = O(\epsilon^2\cdot \eta\norm{\nabla L_t}^2) \le \frac{\eta (2-c)}{4} \norm{P_{u_t,\nabla S_t}^\perp \nabla L_t}^2.
    \end{align*}
    Therefore,
    \begin{align*}
        L(\thetad_{t+1})
        &\le L(\theta_t^\dagger) - \frac{\eta (2-c)}{4} \norm{P_{u_t,\nabla S_t}^\perp \nabla L_t}^2 = L(\theta_t^\dagger) - \Omega(\eta\norm{P_{u_t,\nabla S_t}^\perp \nabla L_t}^2)
    \end{align*}
    which completes the proof.
\end{proof}

\begin{corollary}
    Let $L^\star = \min_\theta L(\theta)$. Then there exists $t \le \mathscr{T}$ such that
    \begin{align*}
        \norm{P_{u_t,\nabla S_t}^\perp \nabla L_t}^2 \le O\qty(\frac{L(\thetad_0) - L^\star}{\eta \mathscr{T}}).
    \end{align*}
\end{corollary}
\begin{proof}
    Inductively applying \Cref{lem:dagger_descent} we have that there exists an absolute constant $c$ such that
    \begin{align*}
        L^\star \le L(\thetad_\mathscr{T}) \le L(\thetad_0) - c \eta \sum_{t < \mathscr{T}} \norm{P_{u_t,\nabla S_t}^\perp \nabla L_t}^2
    \end{align*}
    which implies that
    \begin{align*}
        \min_{t < \mathscr{T}}\norm{P_{u_t,\nabla S_t}^\perp \nabla L_t}^2 \le \frac{\sum_{t < \mathscr{T}} \norm{P_{u_t,\nabla S_t}^\perp \nabla L_t}^2}{\mathscr{T}} \le O\qty(\frac{L(\thetad_0) - L^\star}{\eta \mathscr{T}}).
    \end{align*}
\end{proof}

\subsection[Proof of Coupling Theorem]{Proof of \Cref{thm:coupling}}

We first require the following three lemmas, whose proofs are deferred to \Cref{sec:aux_proofs}.

\begin{lemma}[2-Step Lemma]\label{lem:2step}
    Let
    \begin{align*}
        r_t := v_{t+2} - \step_{t+1}(\step_t(v_t)).
    \end{align*}
    Assume that $\norm{v_t} \le \epsilon^{-1}\delta$. Then
    \begin{align*}
        \norm{r_t} \le O\qty(\epsilon^2\delta\cdot\max\qty(1, \frac{\norm{v_t}}{\delta})^3).
    \end{align*}
\end{lemma}

\begin{lemma}\label{lem:coupling}
    Assume that there exists constants $c_1,c_2$ such that for all $t \le \mathscr{T}$, $\norm{\vs_t} \le c_2\delta$, $\abs{\xs_t} \ge c_1\delta$. Then, for all $t\le \mathscr{T}$, we have
    \begin{align*}
        \norm{v_t - \vs_t} \le O(\epsilon\delta)
    \end{align*}
\end{lemma}

\begin{lemma}\label{lem:bound_predicted_naive}
For $t \le \mathscr{T}$, $\norm{\vs_t} \le O(\delta)$.
\end{lemma}
With these lemmas in hand, we can prove \Cref{thm:coupling}.
\begin{proof}[Proof of \Cref{thm:coupling}]

First, by \Cref{lem:bound_predicted_naive}, we have $\norm{\vs_t} \le O(\delta)$.

Next, by \Cref{lem:coupling}, we have
\begin{align*}
    \theta_t - \thetad_t = v_t = \vs_t + O(\epsilon\delta).
\end{align*}

Next, we Taylor expand to calculate $S(\theta_t)$:
\begin{align*}
    S(\theta_t) &= S(\thetad_t) + \nabla S_t\cdot v_t + O(\eta\rho_3^2\norm{v_t}^2)\\
    &= 2/\eta + \nabla S_t^\perp\cdot v_t + \nabla S_t \cdot u_t u_t \cdot v_t + O(\eta\rho_3^2\delta^2)\\
    &= 2/\eta + \nabla S_t^\perp\cdot \vs_t + \nabla S_t \cdot u_t u_t \cdot \vs_t + O(\rho_3\epsilon\delta + \eta\rho_3^2\delta^2)\\
    &= 2/\eta + y_t + (\nabla S_t \cdot u_t)x_t + O(\eta^{-1}\epsilon^2).\\
\end{align*}

Finally, we Taylor expand the loss:
\begin{align*}
    L(\theta_t) &= L(\thetad_t) + \nabla L_t\cdot v_t + \frac12v_t^T\nabla^2L_tv_t + O(\rho_3\norm{v_t}^3)\\
    &= L(\thetad_t) + \frac{1}{\eta}x_t^2 + \frac12{v_t^\perp}^T\nabla^2L_t{v^\perp_t} + O(\rho_1\norm{v_t} + \rho_3\norm{v_t}^3)\\
    &= L(\thetad_t) + \frac{1}{\eta}\xs_t^2 + \frac12{\vs_t^\perp}^T\nabla^2L_t{\vs^\perp_t} + O(\eta^{-1}\delta^2\epsilon)\\
    &= L(\thetad_t) + \frac{1}{\eta}\xs_t^2 + O(\eta^{-1}\delta^2\epsilon),
\end{align*}
where the last line follows from \Cref{assume:non_worst}.

\end{proof}

\subsection{Proof of Auxiliary Lemmas}\label{sec:aux_proofs}

\begin{proof}[Proof of \Cref{lem:2step}]
Taylor expanding the update for $\theta_{t+1}$ about $\thetad_t$, we get
\begin{align*}
    \theta_{t+1} &= \theta_t - \eta \nabla L(\theta_t)\\
        &= \theta_t - \eta\nabla L_t - \eta \nabla^2L_t v_t - \frac12\eta\nabla^3L_t(v_t, v_t) + O\qty(\eta\rho_4\norm{v_t}^3)
\end{align*}
Additionally, recall that the update for $\thetad_{t+1}$ is
\begin{align*}
    \thetad_{t+1} &= \thetad_t - \eta P_{\nabla S_t^\perp}^\perp \nabla L_t + O\qty(\epsilon^2\cdot \eta\norm{\nabla L_t}).
\end{align*}
Subtracting the previous 2 equations and expanding out $\nabla^3L(v_t, v_t)$ via the non-worst-case bounds, we obtain
\begin{align*}
    v_{t+1} &= (I - \eta \nabla^2L_t)v_t -\eta(\nabla L_t - P_{\nabla S_t^\perp}^\perp \nabla L_t) - \frac12\eta x_t^2\nabla S_t - \eta x_t\nabla^3L_t(u_t, v^\perp_t) - \frac12\eta\nabla^3L_t(v^\perp_t, v^\perp_t)\\
    &\quad + O\qty(\eta\rho_4\norm{v_t}^3 + \epsilon^2\cdot \eta\norm{\nabla L_t})\\
    &= (I - \eta \nabla^2L_t)v_t - \eta\qty[\frac{\nabla L \cdot \nabla S^\perp}{\norm{\nabla S^\perp}^2}]\nabla S^\perp_t - \frac12\eta x_t^2\nabla S_t - \eta x_t\nabla^3L_t(u_t, v^\perp_t)\\
    &\quad + O\qty(\eta\rho_3\epsilon\norm{v_t}^2 + \eta\rho_4\norm{v_t}^3 + \epsilon^2\cdot \eta\norm{\nabla L_t})\\
    &= (I - \eta \nabla^2L_t)v_t + \eta\nabla S^\perp_t\qty[\frac{\delta_t^2 - x_t^2}{2}] - \frac12\eta x_t^2\nabla S_t\cdot u_t u_t - \eta x_t\nabla^3L_t(u_t, v^\perp_t)\\
    &\quad + O\qty(\epsilon^2\cdot\frac{\norm{v_t}^2}{\delta} + \epsilon^2\cdot\frac{\norm{v_t}^3}{\delta^2} + \epsilon^3\delta)\\
    &= (I - \eta \nabla^2L_t)v_t + \eta\nabla S^\perp_t\qty[\frac{\delta_t^2 - x_t^2}{2}] - \frac12\eta x_t^2\nabla S_t\cdot u_t u_t - \eta x_t\nabla^3L_t(u_t, v^\perp_t)\\
    &\quad + O\qty(\epsilon^2\delta\cdot\max\qty(1, \frac{\norm{v_t}}{\delta})^3)\\
\end{align*}
We would first like to compute the magnitude of $v_{t+1}$.
\begin{align*}
    \norm{v_{t+1}} = O\qty(\norm{v_t} + \eta\rho_3\norm{v_t}^2 + \eta\norm{\nabla L_t} + \epsilon^2\delta\cdot\max\qty(1, \frac{\norm{v_t}}{\delta})^3).
\end{align*}
Observe that by definition of $\epsilon$ and $\delta$, and since $\norm{v_t} \le \epsilon^{-1}\delta$
\begin{align*}
    O(\eta\rho_3\norm{v_t}^2) &\le O\qty(\norm{v_t}\cdot \epsilon^{-1} \eta\rho_3\delta) \le O\qty(\norm{v_t} \cdot \epsilon^{-1}\eta\sqrt{\rho_1\rho_3}) \le O\qty(\norm{v_t})\\
    O(\epsilon^2\delta\cdot\max\qty(1, \frac{\norm{v_t}}{\delta})^3) &\le O\qty(\epsilon^2\delta + \norm{v_t}\cdot \epsilon^2 \cdot (\epsilon^{-1})^2) \le O\qty(\epsilon^2\delta + \norm{v_t}).
\end{align*}
Hence
\begin{align*}
    \norm{v_{t+1}} = O\qty(\norm{v_t} + \eta\norm{\nabla L_t} + \epsilon^2\delta) = O\qty(\norm{v_t} + \epsilon\delta).
\end{align*}
Note that we can bound
\begin{align*}
    \norm{u_{t+1} - u_t}\cdot\norm{v_{t+1}} &= O\qty(\eta^2\rho_3\norm{\nabla L_t}\cdot(\norm{v_t} + \epsilon\delta))\\
    &= O\qty(\epsilon^{2}\cdot(\norm{v_t} + \epsilon\delta))\\
    &\le O\qty(\epsilon^2\cdot \max(\norm{v_t}, \delta)). 
\end{align*}
Therefore, the one-step update in the $u_t$ direction is:
\begin{align*}
    x_{t+1} &= v_{t+1}\cdot u_{t+1}\\
    &= v_{t+1}\cdot u_t + O\qty(\epsilon^2\cdot \max(\norm{v_t}, \delta))\\
    &= -v_t\cdot u_t - \frac12\eta x_t^2 \nabla S_t \cdot u_t - \eta x_t \nabla S_t \cdot v^\perp_t + O\qty(\epsilon^2\cdot \max(\norm{v_t}, \delta) + \epsilon^2\delta\cdot\max\qty(1, \frac{\norm{v_t}}{\delta})^3)\\
    &= -x_t(1 + \eta y_t) - \frac12 \eta x_t^2 \nabla S_t \cdot u_t + O\qty(\epsilon^2\delta\cdot\max\qty(1, \frac{\norm{v_t}}{\delta})^3)\\
    &= -x_t(1 + \eta y_t) - \frac12 \eta x_t^2 \nabla S_t \cdot u_t + O\qty(\epsilon^2\delta\cdot\max\qty(1, \frac{\norm{v_t}}{\delta})^3)\\
    &= -x_t(1 + \eta y_t) - \frac12 \eta x_t^2 \nabla S_t \cdot u_t + O(E_t),
\end{align*}
where we have defined the error term $E_t$ as
\begin{align*}
    E_t := \epsilon^2\delta\cdot\max\qty(1, \frac{\norm{v_t}}{\delta})^3.
\end{align*}
The update in the $v^\perp$ direction is
\begin{align*}
    v_{t+1}^\perp &= P_{u_{t+1}}^\perp\qty[(I - \eta \nabla^2L_t)v_t + \eta\nabla S_t^\perp\qty[\frac{\delta_t^2  - x^2}{2}]] - \frac12\eta x_t^2\nabla S_t \cdot u_t P_{u_{t+1}}^\perp u_t - \eta x_t P_{u_{t+1}}^\perp\nabla^3 L_t(u_t, v^\perp_t)\\
    &\quad + O\qty(\epsilon^2\delta\cdot\max\qty(1, \frac{\norm{v_t}}{\delta})^3)\\
    &= P_{u_{t+1}}^\perp\qty[(I - \eta \nabla^2L_t)P_{u_t}^\perp v_t + \eta\nabla S_t^\perp\qty[\frac{\delta_t^2  - x^2}{2}]] - x_tP_{u_{t+1}}^\perp u_t - \frac12\eta x_t^2\nabla S_t \cdot u_t P_{u_{t+1}}^\perp u_t\\
    &\quad- \eta x_t P_{u_{t+1}}^\perp\nabla^3L_t(u_t, v_t^\perp) + O\qty(\epsilon^2\delta\cdot\max\qty(1, \frac{\norm{v_t}}{\delta})^3)
\end{align*}
First, observe that
\begin{align*}
    \norm{P_{u_{t+1}}^\perp u_t} = \norm{u_t - u_{t+1}u_{t+1}^Tu_t} \le \norm{u_t - u_{t+1}}^2 \le O\qty(\norm{u_t - u_{t+1}})
\end{align*}
Therefore we can control the first of the error terms as
\begin{align*}
    \norm{x_tP_{u_{t+1}}^\perp u_t + \frac12\eta x_t^2\nabla S_t \cdot u_t P_{u_{t+1}}^\perp u_t} &\le O\qty(\norm{u_t - u_{t+1}}\cdot(\norm{v_t} + \eta\rho_3\norm{v_t}^2))\\
    &\le O\qty(\norm{u_t - u_{t+1}}\cdot\norm{v_t})\\
    &\le O\qty(\epsilon^2\norm{v_t}),
\end{align*}
As for the second error term, we can decompose
\begin{align*}
    \norm{\eta x_t P^\perp_{u_{t+1}}\nabla^3 L_t(u_t, v^\perp_t)} &\le \eta\norm{v_t}\qty(\norm{P_{u_{t}}^\perp \nabla^3 L_t(u_t, v_t^\perp)} + \norm{P_{u_t}^\perp - P_{u_{t+1}}^\perp}\norm{\nabla^3 L_t(u_t, v_t^\perp)}).
\end{align*}
By \Cref{assume:non_worst}, we have $\norm{P_{u_t}^\perp\nabla^3L_t(u_t, v_t^\perp)} \le O(\epsilon\rho_3 \norm{v_t})$. Additionally, $\norm{P_{u_t}^\perp - P_{u_{t+1}}^\perp} \le O\qty(\norm{u_t - u_{t+1}})$. Therefore
\begin{align*}
    \norm{\eta x_t P^\perp_{u_{t+1}}\nabla^3 L_t(u_t, v^\perp_t)}&\le O\qty(\epsilon\rho_3 \norm{v_t}\cdot \eta\norm{v_t} + \eta\norm{v_t}\norm{u_{t+1} - u_t}\cdot \rho_3\norm{v_t})\\
    &\le O\qty(\epsilon \eta\rho_3\norm{v_t}^2 + \eta\rho_3\norm{v_t}^2\epsilon^2)\\
    &\le O\qty(\epsilon^2 \frac{\norm{v_t}^2}{\delta} + \epsilon^2\norm{v_t})\\
    &= O\qty(\epsilon^2\delta\cdot\max\qty(1, \frac{\norm{v_t}}{\delta})^3)
\end{align*}
where we used $\eta\rho_3\norm{v_t} = O(1)$. Altogether, we have
\begin{align*}
    v_{t+1}^\perp &= P_{u_{t+1}}^\perp\qty[(I - \eta \nabla^2L_t)P_{u_t}^\perp v_t + \eta\nabla S_t^\perp\qty[\frac{\delta_t^2  - x^2}{2}]] + O\qty(\epsilon^2\delta\cdot\max\qty(1, \frac{\norm{v_t}}{\delta})^3)\\
    &= P_{u_{t+1}}^\perp\qty[(I - \eta \nabla^2L_t)P_{u_t}^\perp v_t + \eta\nabla S_t^\perp\qty[\frac{\delta_t^2  - x^2}{2}]] + O(E_t)
\end{align*}
We next compute the two-step update for $x_t$:
\begin{align*}
    x_{t+2} &= -x_{t+1}(1 + \eta y_{t+1}) - \frac12\eta x_{t+1}^2\nabla S_{t+1}\cdot u_{t+1} + O(E_{t+1})\\
    &= x_t(1 + \eta y_t)(1 + \eta y_{t+1}) + \frac{\eta}{2}\qty(\eta y_t x_t^2\nabla S_t \cdot u_t + x_t^2\nabla S_t\cdot u_t - x_{t+1}^2\nabla S_{t+1}\cdot u_{t+1})\\
    &\quad + O\qty((1 + \eta\rho_3\norm{v_t})E_t + E_{t+1}).
\end{align*}
We previously obtained $\eta\rho_3\norm{v_t} = O(1)$. Furthermore,
\begin{align*}
    E_{t+1} &= \epsilon^2\delta\cdot\max\qty(1, \frac{\norm{v_{t+1}}}{\delta})^3\\
    &= O\qty(\epsilon^2\delta\cdot\max\qty(1, \frac{\norm{v_t}}{\delta} + \epsilon)^3)\\
    &= O\qty(\epsilon^2\delta\cdot\max\qty(1, \frac{\norm{v_t}}{\delta})^3)\\
    &= O\qty(E_t).
\end{align*}
Hence
\begin{align*}
    x_{t+2} &= x_t(1 + \eta y_t)(1 + \eta y_{t+1}) + \frac{\eta}{2}\qty(\eta y_t x_t^2\nabla S_t \cdot u_t + x_t^2\nabla S_t\cdot u_t - x_{t+1}^2\nabla S_{t+1}\cdot u_{t+1}) + O(E_t).
\end{align*}
The first of these two error terms can be bounded as
\begin{align*}
    \abs{\frac12\eta^2y_t x_t^2\nabla S_t \cdot u_t}\le O\qty(\eta^2\rho_3^2\norm{v_t}^3) \le O\qty(\epsilon^2\cdot\frac{\norm{v_t}^3}{\delta^2}).
\end{align*}
As for the second term, we can bound
\begin{align*}
    \abs{\nabla S_{t+1}\cdot u_{t+1} - \nabla S_t\cdot u_t} &\le \abs{u_{t+1}\cdot(\nabla S_{t+1} - \nabla S_t)} + \abs{\nabla S_t\cdot(u_{t+1} - u_t)}\\
    &\le \norm{\nabla S_{t+1} - \nabla S_t} + O(\rho_3)\cdot \norm{u_{t+1} - u_t}\\
    &\le O\qty(\eta^2\rho_3^2\norm{\nabla L_t})\\
    &\le O(\epsilon^2\rho_3)
\end{align*}
Additionally, we have
\begin{align*}
    x_{t+1} = -x_t + O(\eta\rho_3\norm{v_t}^2 + E_t).
\end{align*}
Therefore
\begin{align*}
    \eta\abs{x_{t+1}^2\nabla S_{t+1}\cdot u_{t+1}- x_t^2\nabla S_t\cdot u_t} &\le \eta x_t^2\abs{\nabla S_{t+1}\cdot u_{t+1} - \nabla S_t\cdot u_t} + \eta(x_{t+1}^2 - x_t^2)\abs{\nabla S_{t+1} \cdot u_{t+1}}\\
    &\le O\qty(\eta\rho_3\norm{v_t}^2\cdot\epsilon^2 + \eta\rho_3\norm{v_t}\qty(\eta\rho_3\norm{v_t}^2 + E_t))\\
    &\le O\qty(\epsilon^2\norm{v_t} + \epsilon^2\cdot\frac{\norm{v_t}^3}{\delta^2} + E_t)\\
    &= O\qty(E_t).
\end{align*}

Altogether, the two-step update for $x_t$ is
\begin{align*}
    x_{t+2} &= x_t(1 + \eta y_t)(1 + \eta y_{t+1}) + O\qty(E_t).
\end{align*}
Additionally, the two-step update for $v^\perp_t$ is
\begin{align*}
    v^\perp_{t+2} &= P_{u_{t+2}}^\perp\qty[(I - \eta \nabla^2L_{t+1})P_{u_{t+1}}^\perp v_{t+1} + \eta\nabla S_{t+1}^\perp\qty[\frac{\epsilon_{t+1}^2 - x_{t+1}^2}{2}]] + O(E_{t+1})\\
    &= P_{u_{t+2}}^\perp(I - \eta \nabla^2L_{t+1})P_{u_{t+1}}^\perp (I - \eta \nabla^2L_t)P_{u_t}^\perp v_t + \eta P_{u_{t+2}}^\perp(I - \eta \nabla^2L_{t+1}) P_{u_{t+1}}^\perp\nabla  S_t^\perp\qty[\frac{\delta_t^2 - x_t^2}{2}]\\
    &\quad + \eta P_{u_{t+2}}^\perp\nabla S_{t+1}^\perp\qty[\frac{\epsilon_{t+1}^2 - x_{t+1}^2}{2}] + O(E_t).\\
\end{align*}
Define $\bar v_{t+1} = \step_t(v_t), \bar v_{t+2} = \step_{t+1}(\bar v_t)$, and $\bar x_i = \bar v_i \cdot u_i, \bar y_i = \nabla S^\perp_i \cdot \bar v_i$ for $i \in \{t+1, t+2\}$. By the definition of $\step$, one sees that
\begin{align*}
    \norm{\bar v_{t+1}^\perp - v_{t+1}^\perp} \le O(E_t).
\end{align*}
and
\begin{align*}
    \abs{\bar x_{t+1} - x_{t+1}} &\le \frac12\eta x_t^2\abs{\nabla S_t \cdot u_t} + O(E_t) \le O(\eta\rho_3\norm{v_t}^2 + E_t)
\end{align*}
The update for $x$ after applying $\step$ is
\begin{align*}
    \bar x_{t+2} &= -\bar x_{t+1}(1 + \eta \bar y_{t+1})\\
    &= x_t(1 + \eta y_t)(1 + \eta \bar y_{t+1}).
\end{align*}
Therefore
\begin{align*}
    \abs{x_{t+2} - \bar x_{t+2}} &\le O\qty(\abs{x_t}\eta\abs{y_{t+1} - \bar y_{t+1}}) + O\qty(E_t)\\
    &\le O\qty(\eta\rho_3\norm{v_t}\norm{v^\perp_{t+1} - \bar v^\perp_{t+1}}) + O\qty(E_t)\\
    &\le O\qty(E_t).
\end{align*}
Additionally, the update for $v^\perp$ is
\begin{align*}
    \bar v^\perp_{t+2} &= P_{u_{t+2}}^\perp(I - \eta \nabla^2L_{t+1})P_{u_{t+1}}^\perp (I - \eta \nabla^2L_t)P_{u_t}^\perp v_t + \eta P_{u_{t+2}}^\perp(I - \eta \nabla^2L_{t+1}) P_{u_{t+1}}^\perp\nabla  S_t^\perp\qty[\frac{\delta_t^2 - x_t^2}{2}]\\
    &\quad + \eta P_{u_{t+2}}^\perp\nabla S_{t+1}^\perp\qty[\frac{\epsilon_{t+1}^2 - \bar x_{t+1}^2}{2}].
\end{align*}
Therefore
\begin{align*}
    \norm{v_{t+2}^\perp - \bar v_{t+2}^\perp} &\le O\qty(\eta\norm{\nabla S_{t+1}}(x_{t+1}^2 - \bar x_{t+1}^2) + E_t)\\
    & \le O\qty(\eta\rho_3\norm{v_t}\abs{\bar x_{t+1} - x_{t+1}} + E_t)\\
    &\le O\qty(\eta^2\rho_3^2\norm{v_t}^3 + E_t)\\
    &\le O\qty(\epsilon^2\cdot\frac{\norm{v_t}^3}{\delta^2} + E_t)\\
    &= O\qty(E_t)
\end{align*}
Altogether, we get that
\begin{align*}
    \norm{r_t} \le O\qty(E_t) = O\qty(\epsilon^2\delta\cdot\max\qty(1, \frac{\norm{v_t}}{\delta})^3),
\end{align*}
as desired.
\end{proof}

\begin{proof}[Proof of \Cref{lem:coupling}]
    Define
    \begin{align*}
        w_t =
        \begin{cases}
            0 & t\text{ if is even} \\
            r_{t-1} & t\text{ if is odd}
        \end{cases}
    \end{align*}
    and define the auxiliary trajectory $\widehat v$ by $\widehat v_0 = v_0$ and $\widehat v_{t+1} = \step(\widehat v_t) + w_t$. I first claim that $\widehat v_t = v_t$ for all even $t \le \mathscr{T}$, which we will prove by induction on $t$. The base case is given by assumption so assume the result for some even $t \ge 0$. Then,
    \begin{align*}
        v_{t+2}
        &= \step_{t+1}(\step_t(v_t)) + r_t \\
        &= \step_{t+1}(\step_t(\widehat{v}_t)) + r_t \\
        &= \step_{t+1}(\widehat v_{t+1}) + w_{t+1} \\
        &= \widehat{v}_{t+2}
    \end{align*}
    which completes the induction. 
    
    Next, we will prove by induction that for $t \le \mathscr{T}$, 
    \begin{align*}
        \norm{\widehat v^\perp_t - \vs_t^\perp}, \abs{\widehat x_t - \xs_t} \le O(\epsilon\delta) \le c_2\delta.
    \end{align*}
    By definition, $\widehat v_0 = v_0 = \vs_0$, so the claim is clearly true for $t = 0$. Next, assume the claim holds for $t$. If $t$ is even then $\norm{w_t} = 0$; otherwise $\norm{v_t} \le 2c_2\delta$, and thus
    \begin{align*}
        \norm{w_t} &\le O\qty(\epsilon^2\delta\cdot\max\qty(1, c_2)^3) \le O\qty(\epsilon^2 \delta).
    \end{align*}
    First observe that
    \begin{align*}
        \norm{\widehat v_{t+1}^\perp - \vs_{t+1}^\perp}
        &\le \norm{(I - \eta \nabla^2 L_t)(\widehat v_t^\perp - \vs_t^\perp)} + \frac{\eta \rho_3 \abs{\widehat x_t^2 - \xs_t^2}}{2} + \norm{w_t} \\
        &\le \qty(1 + \eta \abs{\lambda_{min}(\nabla^2 L_t)})\norm{\widehat v_{t}^\perp - \vs_{t}^\perp} + O(\epsilon) \cdot \abs{\widehat x_t - \xs_t} + O\qty(\epsilon^2 \delta)\\
        &\le \qty(1 + \eta \abs{\lambda_{min}(\nabla^2 L_t)})\norm{\widehat v_{t}^\perp - \vs_{t}^\perp} + O(\epsilon\delta) \cdot \abs{\frac{\widehat x_t - \xs_t}{\xs_t}} + O\qty(\epsilon^2 \delta)
    \end{align*}
    Next, note that
    \begin{align*}
        \frac{\widehat x_{t+1}}{\xs_{t+1}}
        &= \frac{(1 + \eta \widehat y_t)\widehat x_t + O(\epsilon^2\delta)}{(1 + \eta \ys_t)\xs_t + O(\epsilon^2\delta)} \\
        &= \frac{(1 + \eta \ys_t)\widehat x_t + O(\epsilon^2\delta) + O(\epsilon)\cdot \norm{\widehat v_t^\perp - \vs_t^\perp}}{(1 + \eta \ys_t)\xs_t + O(\epsilon^2\delta)} \\
        &= \frac{\widehat x_t}{\xs_t} + O\qty(\epsilon^2 + \frac{\epsilon}{\delta}\norm{\widehat v_t^\perp - \vs_t^\perp}).
    \end{align*}
    Therefore
    \begin{align*}
        \abs{\frac{\widehat x_{t+1} - \xs_{t+1}}{\xs_{t+1}}} \le \abs{\frac{\widehat x_t - \xs_t}{\xs_t}} + O(\epsilon^2  + \frac{\epsilon}{\delta}\norm{\widehat v_t^\perp - \vs_t^\perp}).
    \end{align*}
    Let $d_t = \max\qty(\norm{\widehat v_{t}^\perp - \vs_{t}^\perp}, \delta\abs{\frac{\widehat x_t - \xs_t}{\xs_t}})$. Then
    \begin{align*}
        \norm{\widehat v_{t+1}^\perp - \vs_{t+1}^\perp} &\le (1 + \eta \abs{\lambda_{min}(\nabla^2 L_t)} + O(\epsilon))d_t + O(\epsilon^2\delta)\\
        \delta \abs{\frac{\widehat x_{t+1} - \xs_{t+1}}{\xs_{t+1}}} &\le (1 + O(\epsilon))d_t + O(\epsilon^2\delta).
    \end{align*}
    Therefore
    \begin{align*}
         d_{t+1} &\le (1 + \eta \abs{\lambda_{min}(\nabla^2 L_t)} + O(\epsilon))d_t + O(\epsilon^2\delta)\\
         &\le (1 + O(\epsilon))d_t + O(\epsilon^2\delta),
    \end{align*}
    so for $t \le \mathscr{T}$ we have $d_{t+1} \le O(\epsilon\delta)$. Therefore 
    \begin{align*}
        \norm{\widehat v^\perp_{t+1} - \vs_{t+1}^\perp}, \abs{\widehat x_{t+1} - \xs_{t+1}} \le O(\epsilon\delta) \le c_2\delta,
    \end{align*}
    so the induction is proven. Altogether, we get $\norm{\widehat v_t - \vs_t} \le O(\epsilon\delta)$ for all such $t$, as desired.
\end{proof}

\begin{proof}[Proof of \Cref{lem:bound_predicted_naive}]
Recall that
\begin{align*}
    x^*_{t+1} = - (1 + \eta y^*_t)x^*_t \qand y^*_{t+1} = \eta\sum_{s = 0}^t \beta_{s \to t}\qty[\frac{\delta_s^2 - {x^*_s}^2}{2}],
\end{align*}
Since $t \le \frac{1}{\eta \max_t \abs{\lambda_{min}(\nabla^2 L_t)}}$, we have that $\beta_{s \to t} = O(\rho_3^2)$, and thus
\begin{align*}
    \ys_t \le O(\rho_3^2)t\eta\delta^2 = O(\sqrt{\rho_1\rho_3}).
\end{align*}
Therefore
\begin{align*}
    \abs{\xs_{t+1}} = (1 + \eta \ys_t)\abs{\xs_t} \le (1 + O(\epsilon))\abs{\xs_t}.
\end{align*}
Since $t \le O(\epsilon^{-1})$, $\abs{\xs_t}$ grows by at most a constant factor, and thus $\abs{\xs_t} \le O(\delta)$. Finally, recall that
\begin{align*}
    \vs^\perp_{t+1} = \eta\sum_{s = 0}^t P_{u_{t+1}}^\perp \qty[\prod_{k = t}^{s+1} A_k] \nabla S_s^\perp \qty[\frac{\delta_s^2 - {x^*_s}^2}{2}].
\end{align*}
By the triangle inequality,
\begin{align*}
    \norm{\vs^\perp_{t+1}} \le O(\eta t\rho_3 \delta^2) \le O(\delta).
\end{align*}
Therefore $\norm{\vs_t} \le O(\delta)$.
\end{proof}

\end{document}